%% file: main.tex
\documentclass{article}
\usepackage{iclr2027_conference,times}
\iclrfinalcopy
\usepackage[T1]{fontenc}
\usepackage{amsmath,amssymb,amsthm}
\usepackage{graphicx,booktabs,array,multirow}
\usepackage{algorithm,algorithmic}
\usepackage{microtype}
\usepackage{hyperref,url}
\usepackage{xcolor}
\newcommand{\rev}[1]{#1}
\newcommand{\revcol}{}
\newcommand{\abstractcode}{\par\smallskip{\centering\small Code: \href{https://github.com/ApRL-IITKGP/direction-conditioned-policy}{\texttt{github.com/ApRL-IITKGP/direction-conditioned-policy}}\par}}
\newcommand{\codeurl}{https://github.com/ApRL-IITKGP/direction-conditioned-policy}
\hypersetup{hidelinks,
  pdftitle={Direction-Conditioned Policies for Online Goal-Conditioned Reinforcement Learning},
  pdfauthor={S K Swaminathan, Damiya Gondha, Theyanesh Eswaramoorthy Rajahkrishnan, Aritra Hazra}}
\newcommand{\R}{\mathbb{R}}
\newcommand{\E}{\mathbb{E}}
\newcommand{\norm}[1]{\left\lVert#1\right\rVert}
\newtheorem{proposition}{Proposition}
\title{Direction-Conditioned Policies for Online Goal-Conditioned Reinforcement Learning}
\author{{\fontsize{9}{11}\selectfont\bfseries S K Swaminathan\textsuperscript{1}\thanks{Correspondence to: S K Swaminathan \texttt{<swami2004@kgpian.iitkgp.ac.in>}.}, Damiya Gondha\textsuperscript{1}\thanks{Equal contribution.}, Theyanesh Eswaramoorthy Rajahkrishnan\textsuperscript{2}\footnotemark[2] \& Aritra Hazra\textsuperscript{1}}\\[0.5ex]
\textsuperscript{1}Department of Computer Science and Engineering\\
\textsuperscript{2}Department of Mechanical Engineering\\
Indian Institute of Technology Kharagpur, India}
\begin{document}
\maketitle
\fancyhead{}
\renewcommand{\headrulewidth}{0pt}
\suppressfloats[t]
\input{abstract}
\input{sections/1_introduction}
\input{sections/2_related_work}
\input{sections/3_methodology}
\input{sections/4_experiments}
\input{sections/5_conclusion}
\input{statements}
\bibliography{refs}
\bibliographystyle{iclr2027_conference}
\clearpage
\appendix
\input{appendix}
\end{document}

%% file: abstract.tex
\begin{abstract}
Contrastive Reinforcement Learning (CRL) learns representations that estimate goal reachability, yet its policy remains conditioned on raw goals and therefore does not directly exploit the geometry encoded by its critic. We introduce Direction-Conditioned Policies (DCP), a method built around a small modification to CRL: DCP selects previously visited states as waypoints during online training and conditions the policy on their direction and distance in representation space. At deployment, DCP applies the same interface directly to the final goal, requiring neither waypoint selection nor planning. Across nine navigation and manipulation tasks, DCP attains higher final success rates than CRL on seven tasks and spends more time near the goal on seven. \rev{Ablations show that most of the gains come from conditioning on direction and distance, while the waypoints let DCP learn on every task, including those where the goal is rarely reached early in training.} Controlled maze experiments further show that DCP captures shortest-path geometry more accurately and that the supplied direction causally influences the actor's behavior. \rev{Finally, we show} that learned candidate generation improves goal reaching in two controlled mazes.
\abstractcode
\end{abstract}

%% file: sections/1_introduction.tex
\section{Introduction}
\begin{figure}[t]
\centering
\includegraphics[width=0.65\linewidth]{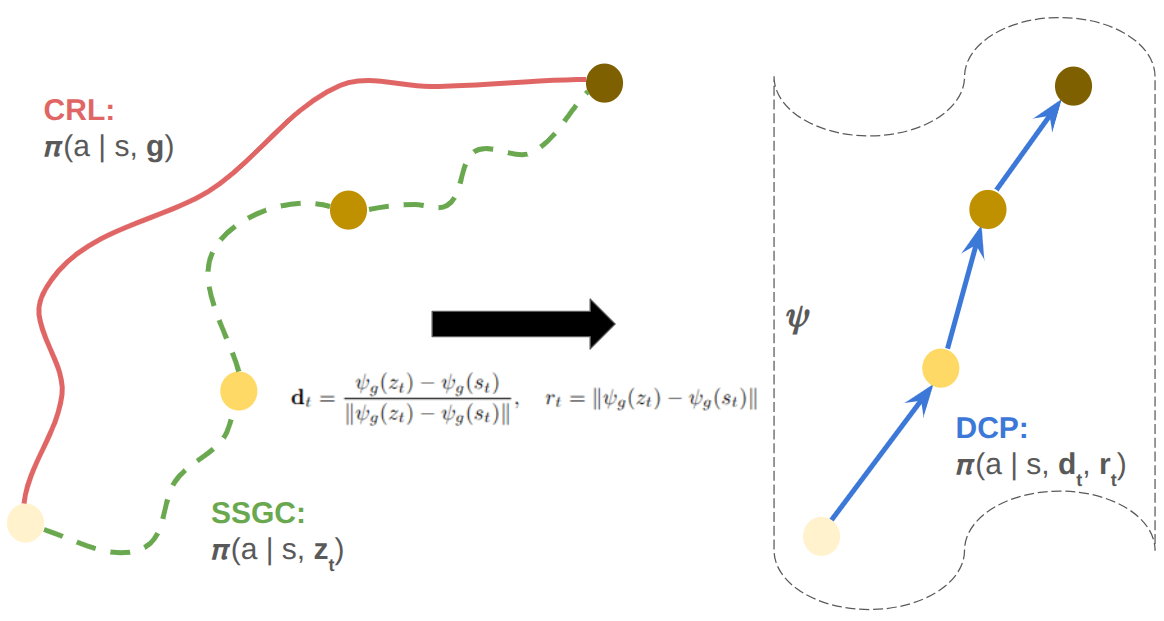}
\caption{Schematic of the policy inputs. CRL uses the raw goal, SSGC uses a selected waypoint, and DCP uses the waypoint's direction and distance in representation space. At deployment, DCP replaces the waypoint with the final goal. Here $\psi(s)$ abbreviates the encoding of the state's achieved-goal coordinates.}
\label{fig:method}
\end{figure}
When learning to achieve goals through online interaction, an agent has to discover useful behavior before it has much successful experience to draw on. Contrastive reinforcement learning (CRL; \citealp{eysenbach2022crl}) makes use of the states the agent reaches along the way. Its critic learns to distinguish achieved future goals from other goals, and the policy learns to choose actions that receive a higher critic score. This gives the agent a learning signal even when reaching the requested goal is rare. Recent work has made this approach increasingly effective through stabilization techniques, accelerated simulation, deeper networks, and action-sequence conditioning~\citep{zheng2023stabilizing, bortkiewicz2025jaxgcrl,wang2025depth,korniak2026chunks}.

However, the policy still receives the goal in raw coordinates. Consider a goal on the opposite side of a maze wall: the coordinate difference points through the wall, even though reaching the goal requires moving away from it first. The critic's representation is learned from trajectories and can capture this route structure. This raises a simple question: can we use that representation to express the goal to the policy? Doing so would give the policy direct access to a geometric relation that the critic already learns. The difficulty is that the representation itself is being learned online. Early in training, a distant goal may lie outside the agent's experience, making its learned direction unreliable. We therefore consider two parts of the problem together: how to express a target to the policy, and which targets to pursue while collecting experience. Previously visited states are natural candidate waypoints \citep{eysenbach2019search, ecoffet2021first}. Scoring these candidates against the final goal yields intermediate targets whose representations have been trained on collected experience.

We propose Direction-Conditioned Policies (DCP), a method built around a small change to the CRL actor. We encode the current achieved goal and the target using the critic's goal encoder, then supply the normalized displacement and its magnitude to the actor, along with the full physical state. During collection, the target is a scored visited waypoint. During hindsight updates, it is a future achieved state. At deployment, it is the requested final goal. In this way, the same policy input supports all three stages, and the waypoint pool and scorer can be removed at deployment. Figure~\ref{fig:method} illustrates these inputs.

Direction-conditioned policies and embedding differences have been used in skill learning and offline goal reaching~\citep{park2024metra,park2024hilp,devin2019plan}. We investigate their use within online CRL, where the policy and representation develop together through the agent's own experience. \rev{Unlike these methods, DCP adds no separate representation objective and uses the critic's own goal encoder.} Our contributions are:
\begin{itemize}
\item An online CRL method that uses the critic's goal representation to condition the actor on direction and distance, with deployment requiring only the actor and goal encoder.
\item A five-seed comparison across nine tasks, with higher final mean success and time near goal than CRL on seven tasks each\rev{, together with a nine-task ablation that separates the actor input from waypoint collection and a parameter-matched raw-goal control}.
\item Controlled experiments that test route geometry, the actor's use of direction, and waypoint selection, including improvements from learned generation in restrictive mazes.
\end{itemize} These results also reveal a remaining difficulty: improving selected waypoints alone is insufficient to address a more disconnected, dynamic loco-manipulation task like AntSoccer. \rev{On this task, collecting directly toward the final goal closes the gap to CRL.}

Section~\ref{sec:related} discusses related work, Section~\ref{sec:method} describes DCP, Section~\ref{sec:results} presents the experiments, and Section~\ref{sec:conclusion} concludes. The appendices provide implementation details and additional analyses.

%% file: sections/2_related_work.tex
\section{Related Work}
\label{sec:related}
\paragraph{Contrastive representations for control.}
CRL derives a goal-reaching objective from contrastive prediction of future states~\citep{eysenbach2022crl}. JaxGCRL makes large online comparisons practical through accelerated simulation~\citep{bortkiewicz2025jaxgcrl}. Depth scaling~\citep{wang2025depth} and action-sequence conditioning~\citep{korniak2026chunks} demonstrate that changes to architecture or critic inputs can substantially alter the capabilities learned with this objective. Single-goal CRL also demonstrates directed exploration before the first successful trial, without intermediate subgoals~\citep{liu2025single}. We ask how much of the learned representation can be reused by the policy itself. DCP supplies relative goal-encoder features to the actor while retaining the baseline network depth and single-action policy.

\paragraph{Direction-conditioned skills and goal representations.}
HILP learns Hilbert representations and directional policies from offline data~\citep{park2024hilp}; METRA learns metric-aware representations and skills through online unsupervised exploration~\citep{park2024metra}. Plan arithmetic uses compositional differences and sums in learned plan representations~\citep{devin2019plan}. Dual Goal Representations encode goals through temporal distances from other states and improve offline goal reaching~\citep{park2026dual}. Relative targets also appear in FeUdal Networks, whose worker follows latent directions from a manager \citep{vezhnevets2017feudal}, HIRO's state-offset subgoals \citep{nachum2018data}, LSD's latent skill directions \citep{park2022lipschitz}, and HIQL's latent subgoal actions \citep{park2023hiql}; learned temporal distances \citep{hartikainen2020dynamical, myers2024distances} and quasimetrics
\citep{wang2023qrl, liu2023metric, myers2025offline} supply the underlying geometry. We bring learned goal relations to an online contrastive actor--critic, retaining its actor objective and learning the representation from the experience the direction-conditioned policy collects.

\paragraph{Interpolation and waypoint substitution.}
Contrastive representations can support inference over intermediate states under specific modeling assumptions~\citep{eysenbach2024interpolation}. Planning invariance examines when an on-path waypoint and a final goal induce the same action~\citep{myers2025horizon}. It arises for DCP, which trains on achieved targets but executes toward the final goal, and for flat policies bootstrapped from subgoal-conditioned ones \citep{zhou2025flattening}. We test whether on-path substitution preserves the action and how reliably the scorer selects such waypoints. 

%% file: sections/3_methodology.tex
\section{Methodology}
\label{sec:method}
\subsection{Contrastive goal-conditioned learning}
Let $s\in\mathcal S$ denote the full state, $a\in\mathcal A$ an action, and $g\in\mathcal G$ a goal. A known projection $q:\mathcal S\rightarrow\mathcal G$ extracts the achieved-goal coordinates, such as the robot's position in navigation or the object's position in manipulation. Success occurs when $q(s)$ lies within the task's goal region. CRL learns a state--action encoder $\phi_\theta(s,a)$ and a goal encoder $\psi_\theta(g)$, both mapping to $\R^k$. We use the JaxGCRL implementation~\citep{bortkiewicz2025jaxgcrl} with negative Euclidean-distance energy,
\begin{equation}
 f_\theta(s,a,g)=-\norm{\phi_\theta(s,a)-\psi_\theta(g)}_2.
 \label{eq:energy}
\end{equation}
For a minibatch of state--action pairs and their sampled future achieved goals, write $\ell_{ij}=f_\theta(s_i,a_i,g_j^+)$. Backward InfoNCE distinguishes the matching state--action pair for each goal:
\begin{equation}
 \mathcal L_{\mathrm{NCE}}(\theta)=-\frac{1}{B}\sum_{i=1}^B
 \log\frac{\exp(\ell_{ii})}{\sum_{j=1}^B\exp(\ell_{ji})}.
 \label{eq:infonce}
\end{equation}
We retain the baseline's log-sum-exp regularization, replay sampling, and entropy-regularized actor objective. Appendix~\ref{app:implementation} gives implementation details.

\subsection{A relative goal input for the actor}
We want the actor to receive the target relative to the current state in the space learned by the critic. We use the same goal encoder for both positions. For a target $y\in\mathcal G$, define
\begin{equation}
 v_\theta(s,y)=\psi_\theta(y)-\psi_\theta(q(s)),\qquad
 r_\theta(s,y)=\norm{v_\theta(s,y)}_2,\qquad
 d_\theta(s,y)=\frac{v_\theta(s,y)}{r_\theta(s,y)+\epsilon},
 \label{eq:interface}
\end{equation}
where $\epsilon=10^{-8}$ prevents division by zero. DCP replaces the raw-goal actor input with $[s,d_\theta(s,y),r_\theta(s,y)]$. The direction specifies where the target lies in the learned space, and the radius gives its distance. \rev{The learned space is also where the critic scores the actor's actions (Eq.~\ref{eq:energy}).} We retain the full physical state because the action also depends on velocities and body configuration. For example, moving an object toward the same goal requires a different approach depending on which side of the object the robot occupies.

For nonzero displacement, direction and radius jointly recover the displacement up to numerical stabilization. A waypoint and the final goal can share a direction while differing in distance; the actor receives both quantities.

\subsection{Collection, learning, and deployment}
\paragraph{Collecting experience with visited targets.}
To provide collection targets as the representation develops, we maintain a pool $\mathcal P$ of $P=512$ states drawn from recently collected trajectories. Every 25 environment steps, it samples $N=32$ candidate indices uniformly with replacement and selects
\begin{equation}
 z^*\in\arg\max_{z\in\mathcal C\subseteq\mathcal P}
 \psi_\theta(q(z))^\top\psi_\theta(g).
 \label{eq:scorer}
\end{equation}
The target $y=q(z^*)$ is held until the next refresh, while its direction relative to the current state is recomputed at each action. The inner product ranks candidates using their embedding alignment and magnitude relative to the final goal. This is a simple way to reuse the learned representation for selection. The candidates come from visited states; whether the policy can reach a selected state from its current position is examined in our experiments.

\paragraph{Learning from achieved future states.}
For learning, each replayed trajectory supplies examples of goals that were actually achieved. We sample a future state and use $y=q(s_{t+\Delta})$ as the actor's target. With $c_\theta(s,y)=[d_\theta(s,y),r_\theta(s,y)]$, the actor minimizes
\begin{equation}
 \mathcal L_\pi(\eta)=\E\left[
 \alpha\log\pi_\eta(a\mid s,c_\theta(s,y))-f_\theta(s,a,y)
 \right],\quad a\sim\pi_\eta(\cdot\mid s,c_\theta(s,y)).
 \label{eq:actor}
\end{equation}
The actor loss updates $\eta$ with the encoder parameters held fixed. The encoders learn through the contrastive objective. As the policy changes the trajectories it collects, it also changes the data from which the representation learns. The entropy coefficient is learned as in the baseline~\citep{haarnoja2018sac}.

\begin{table}[t]
\centering
\caption{Targets used for collection, actor updates, and deployment. Each stage constructs the same direction-and-distance input from its target.}
\label{tab:stages}
\begin{tabular}{lll}
\toprule
Stage & Target $y$ & Policy input \\
\midrule
Online collection & Scored visited waypoint $q(z^*)$ & $[s,d_\theta(s,y),r_\theta(s,y)]$ \\
Actor update & Achieved future state $q(s_{t+\Delta})$ & $[s,d_\theta(s,y),r_\theta(s,y)]$ \\
Deployment & Requested final goal $g$ & $[s,d_\theta(s,g),r_\theta(s,g)]$ \\
\bottomrule
\end{tabular}
\end{table}

\paragraph{Direct execution at deployment.}
At evaluation, we encode the requested final goal and use it directly as $y$ in Equation~\ref{eq:interface}. The actor then executes toward this goal at every step. This removes the pool and scorer, leaving the goal encoder and actor. Table~\ref{tab:stages} summarizes the three stages. Section~\ref{sec:substitution} measures how waypoint-to-goal substitution affects the actor's output.

%% file: sections/4_experiments.tex
\section{Experiments}
\label{sec:results}
\subsection{Does DCP improve online goal reaching?}
Figure~\ref{fig:environments} shows all nine tasks. The Ant and Humanoid navigation tasks require moving the robot to a target position around walls. AntPush additionally requires moving an obstructing block to reach the target position. AntSoccer requires moving a ball to the target. The three Pusher tasks require a robot arm to move a cylinder and differ in target placement. Appendix~\ref{app:tasks} gives their goal distributions, success conditions, and training budgets.

\begin{figure}[t]
\centering
\includegraphics[width=0.75\linewidth]{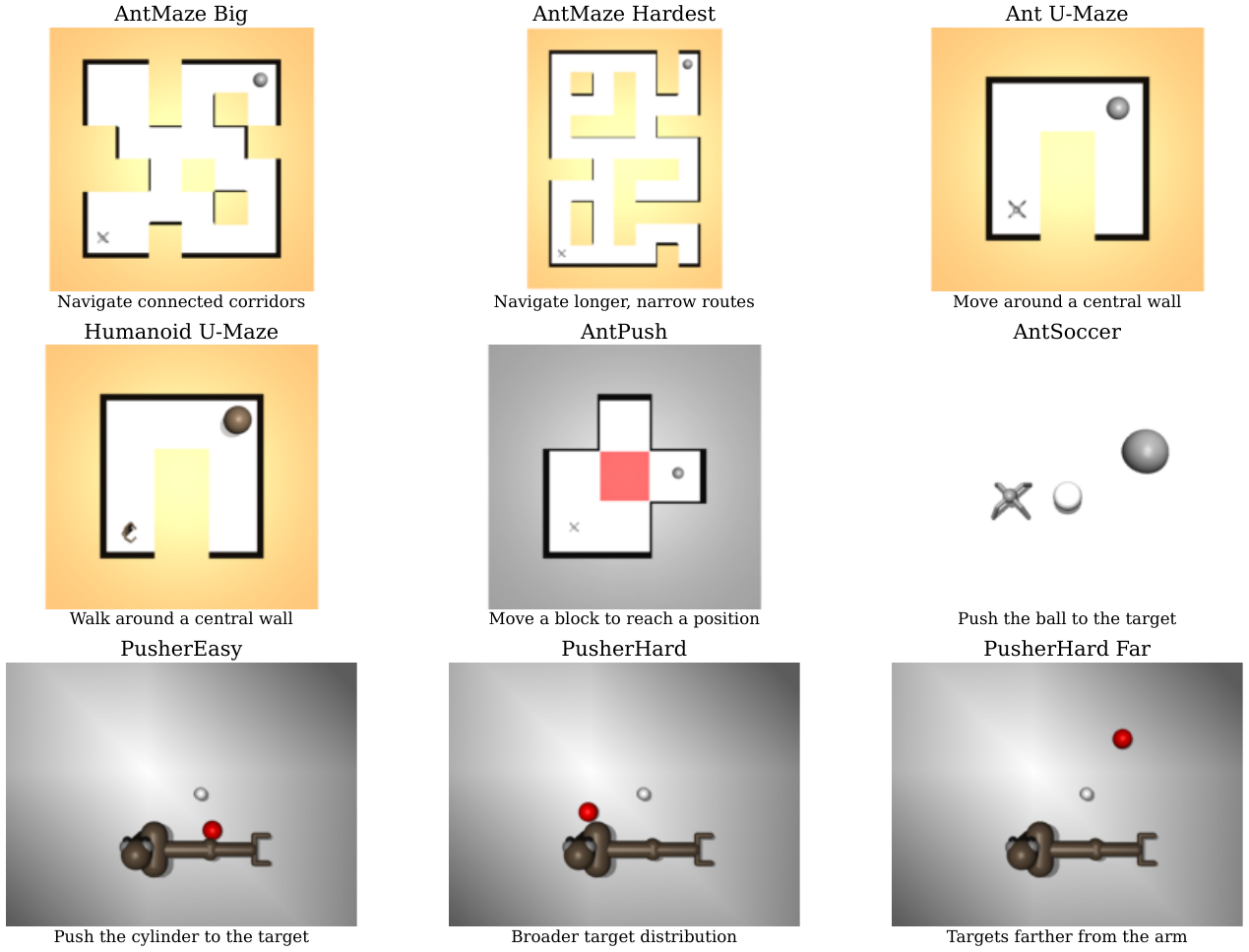}
\caption{The nine benchmark tasks, rendered from their simulation models. The three Pusher tasks share the same arm and cylinder; their target distributions differ. Eight tasks follow JaxGCRL, while PusherHard Far changes PusherHard's target distribution.}
\label{fig:environments}
\end{figure}
We evaluate DCP, CRL, and an ablation over CRL, which we term Self-Subgoal Conditioning (SSGC), on nine Brax-based navigation and manipulation tasks~\citep{freeman2021brax,bortkiewicz2025jaxgcrl}. SSGC uses the same waypoint selection as DCP and supplies raw waypoint coordinates to the actor. \rev{We also evaluate Direct-DCP, an ablation that keeps DCP's actor input but collects experience toward the final goal, without the pool or scorer.} All methods share task-specific interaction budgets, critic objectives, and principal optimization settings\rev{; no method receives separate hyperparameter tuning}. DCP's larger input increases actor parameters by 11--21\%, entirely in the first layer. Hidden widths, depths, action heads, and critic parameter counts are unchanged; the total actor--critic increase is 4.15--6.40\% (Appendix~\ref{app:parameter-counts}). \rev{Section~\ref{sec:factorial} compares DCP with Direct-DCP and with a raw-goal actor of matched parameter count.}

Each method is trained with five seeds. Episodes last 1,000 steps. AntPush, PusherHard, and Humanoid U-Maze use 100M training steps; the other tasks use 50M. We report \emph{success rate}, the percentage of evaluation episodes entering the goal region at least once, and \emph{time near goal}, the mean number of steps spent inside it. We use both because a policy can reach the goal and then move away from it. Final values refer to the last checkpoint. Figure~\ref{fig:curves} gives the learning curves; Appendix~\ref{app:tasks} specifies task settings and Appendix~\ref{app:curves} reports intermediate checkpoints.

\begin{figure}[t]
\centering
\includegraphics[width=\linewidth]{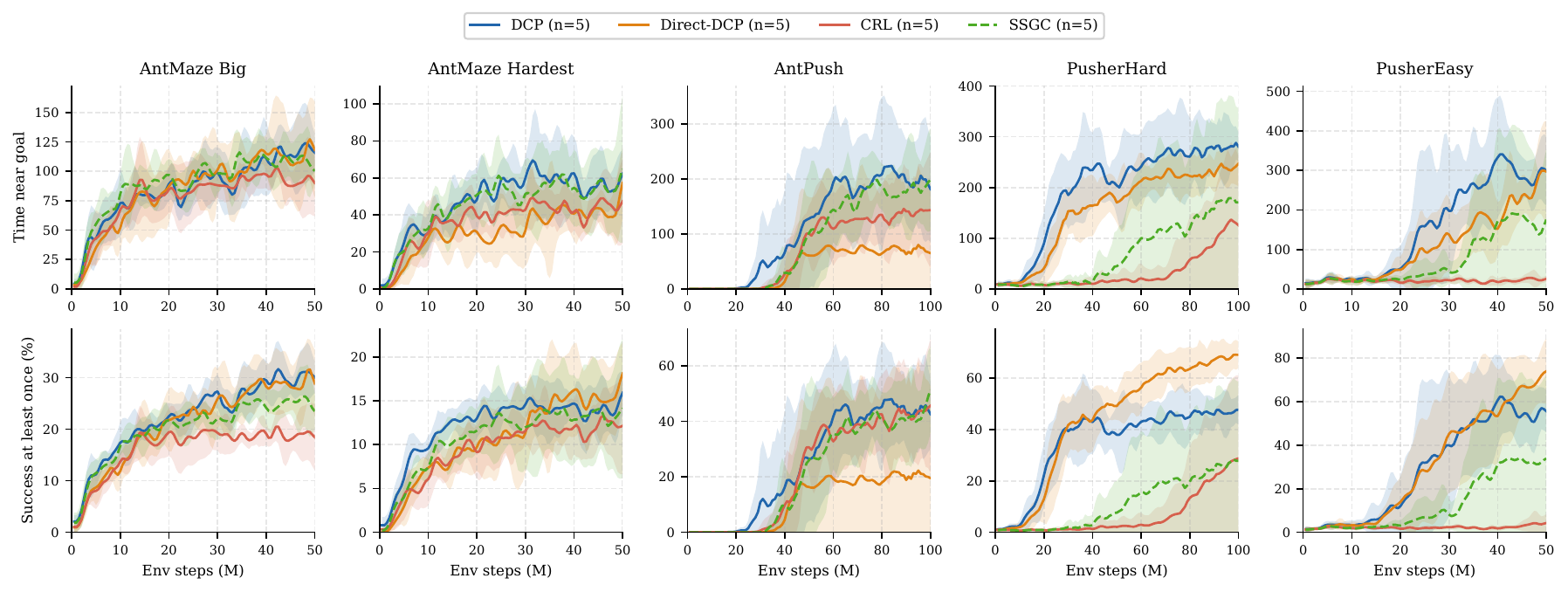}
\includegraphics[width=0.9\linewidth]{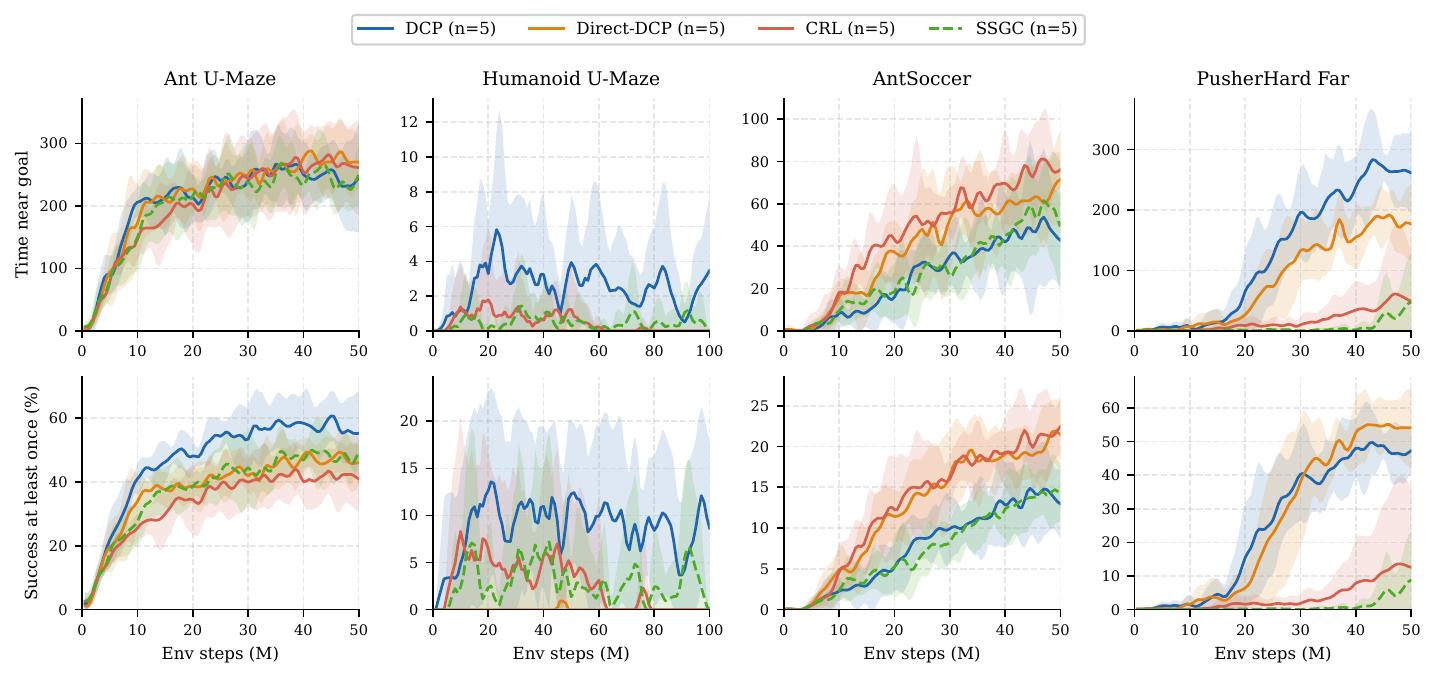}
\caption{Online learning across all nine tasks\rev{, including Direct-DCP}. Each task shows time near goal (upper panel) and success at least once (lower panel). Lines are means over five seeds; shading denotes 95\% Student-$t$ confidence intervals. Budgets are matched across methods within each task. The comparison includes AntSoccer, where CRL outperforms DCP, and Ant U-Maze, where the final two metrics favor different methods.}
\label{fig:curves}
\end{figure}

\begin{table}[t]
\centering\small
\caption{Final evaluation across nine tasks, five seeds per method. Entries are means with reported 95\% confidence-interval half-widths. Bold marks the largest mean in each task and metric. Time near goal is measured in steps out of 1,000. \rev{Direct-DCP uses DCP's actor input and collects toward the final goal.}}
\label{tab:main}
\setlength{\tabcolsep}{1.6pt}
\resizebox{\linewidth}{!}{%
\begin{tabular}{lrrrrrrrr}
\toprule
& \multicolumn{4}{c}{Success rate (\%)} & \multicolumn{4}{c}{Time near goal} \\
\cmidrule(lr){2-5}\cmidrule(lr){6-9}
Task & CRL & SSGC & DCP & \rev{Direct} & CRL & SSGC & DCP & \rev{Direct} \\
\midrule
AntMaze Big & $18.3\!\pm\!6.5$ & $23.3\!\pm\!3.0$ & $\mathbf{29.9}\!\pm\!2.9$ & \rev{$28.4\!\pm\!6.7$} & $89.2\!\pm\!28.0$ & $99.8\!\pm\!9.0$ & $115.3\!\pm\!14.6$ & \rev{$\mathbf{116.5}\!\pm\!42.2$} \\
AntMaze Hardest & $12.2\!\pm\!5.6$ & $14.2\!\pm\!8.0$ & $16.1\!\pm\!2.3$ & \rev{$\mathbf{18.4}\!\pm\!3.4$} & $48.1\!\pm\!23.7$ & $\mathbf{66.0}\!\pm\!41.4$ & $63.2\!\pm\!13.3$ & \rev{$59.0\!\pm\!19.1$} \\
Ant U-Maze & $40.8\!\pm\!10.2$ & $49.8\!\pm\!4.2$ & $\mathbf{55.0}\!\pm\!13.3$ & \rev{$46.2\!\pm\!7.8$} & $261.3\!\pm\!77.2$ & $252.7\!\pm\!52.1$ & $238.1\!\pm\!76.2$ & \rev{$\mathbf{269.9}\!\pm\!39.9$} \\
Humanoid U-Maze & $0.0\!\pm\!0.0$ & $0.0\!\pm\!0.0$ & $\mathbf{8.4}\!\pm\!9.6$ & \rev{$0.0\!\pm\!0.0$} & $0.0\!\pm\!0.0$ & $0.0\!\pm\!0.0$ & $\mathbf{3.6}\!\pm\!4.5$ & \rev{$0.0\!\pm\!0.0$} \\
AntPush & $45.6\!\pm\!24.1$ & $\mathbf{51.4}\!\pm\!15.5$ & $42.7\!\pm\!12.6$ & \rev{$19.5\!\pm\!33.2$} & $143.4\!\pm\!105.6$ & $\mathbf{199.1}\!\pm\!93.3$ & $181.7\!\pm\!95.6$ & \rev{$65.2\!\pm\!113.1$} \\
PusherEasy & $4.4\!\pm\!4.1$ & $34.1\!\pm\!38.5$ & $55.4\!\pm\!9.6$ & \rev{$\mathbf{74.2}\!\pm\!15.2$} & $26.2\!\pm\!8.4$ & $180.7\!\pm\!183.4$ & $\mathbf{297.8}\!\pm\!90.9$ & \rev{$296.2\!\pm\!130.9$} \\
PusherHard & $28.8\!\pm\!32.4$ & $27.8\!\pm\!30.7$ & $47.7\!\pm\!5.8$ & \rev{$\mathbf{69.0}\!\pm\!5.2$} & $125.7\!\pm\!136.7$ & $170.5\!\pm\!186.0$ & $\mathbf{280.9}\!\pm\!28.5$ & \rev{$246.9\!\pm\!44.0$} \\
PusherHard Far & $12.5\!\pm\!24.0$ & $9.0\!\pm\!15.1$ & $47.4\!\pm\!5.7$ & \rev{$\mathbf{54.1}\!\pm\!12.0$} & $48.4\!\pm\!80.4$ & $49.0\!\pm\!82.3$ & $\mathbf{260.7}\!\pm\!66.8$ & \rev{$176.1\!\pm\!62.5$} \\
AntSoccer & $\mathbf{22.6}\!\pm\!3.4$ & $14.1\!\pm\!3.6$ & $12.9\!\pm\!4.3$ & \rev{$21.4\!\pm\!4.7$} & $\mathbf{76.4}\!\pm\!7.7$ & $48.2\!\pm\!31.9$ & $42.4\!\pm\!22.1$ & \rev{$71.7\!\pm\!24.0$} \\
\bottomrule
\end{tabular}}
\end{table}

\textbf{DCP improves final mean success and time near goal over CRL on seven tasks each} (Table~\ref{tab:main}). The largest absolute success improvements occur on PusherEasy ($+51.0$ percentage points), PusherHard Far ($+34.9$), and PusherHard ($+18.9$). DCP also improves success on all four navigation tasks\rev{, including Humanoid U-Maze, where neither CRL nor SSGC reaches the goal}. Compared with SSGC, DCP has higher final mean success on seven tasks.

Looking at both metrics makes the task-level differences clearer. On Ant U-Maze, DCP reaches the goal in more episodes but spends fewer steps there than CRL. On AntPush, its final success is lower while its time near goal is higher; SSGC has the largest final means on both metrics. \rev{On AntSoccer, CRL has the higher final means on both measures; Section~\ref{sec:factorial} shows that this gap comes from waypoint collection.}

\subsection{What do scoring and direction conditioning contribute?}
\label{sec:factorial}
DCP changes both the collection target and the actor input. To understand their individual effects, we compare all four combinations: CRL, direct-goal direction conditioning \rev{(Direct-DCP)}, SSGC, and DCP. The actor receives either raw coordinates or direction and distance, and collection uses either the final goal or a scored waypoint. We run each combination with five seeds. All variants train for 10M steps in an open room and a U-maze. Within each actor family, scoring comparisons match architecture and random-number handling; direct-goal DCP retains the scoring computation but ignores its selected target. Table~\ref{tab:factorial} reports mean evaluation success over the training run, summarizing learning over the full training budget.

\begin{table}[t]
\centering
\caption{Separating collection targets from actor inputs. Mean success over 20 evaluation checkpoints, averaged over five seeds; $\pm$ denotes a 95\% Student-$t$ confidence-interval half-width across seeds. Higher is better.}
\label{tab:factorial}
\begin{tabular}{llll}
\toprule
Method & Collection target & Open room & U-maze \\
\midrule
CRL & Final goal & $0.956\pm0.037$ & $0.846\pm0.153$ \\
\rev{Direct-DCP} & Final goal & $0.974\pm0.010$ & $0.885\pm0.059$ \\
SSGC & Scored waypoint & $0.866\pm0.036$ & $0.851\pm0.040$ \\
DCP & Scored waypoint & $0.869\pm0.033$ & $0.854\pm0.066$ \\
\bottomrule
\end{tabular}
\end{table}

\textbf{Direction conditioning without waypoint selection has the largest mean in both settings.} Scoring reduces open-room performance in every seed for both actor inputs; adding scoring to the direction-conditioned policy changes mean success by $-0.105$, with paired 95\% interval $[-0.141,-0.068]$. In the U-maze, the corresponding difference is $-0.032$ with interval $[-0.125,0.061]$. Here, adding waypoint selection can slow learning. \rev{When every goal can be reached by moving straight toward it, as in the open room, waypoints only redirect collection away from the evaluated goals.} Actor parameters increase from 68,612 to 84,740 (23.5\%); the benchmark actors have larger state inputs and action heads, making the same added input weights a smaller fraction of their parameters.

\rev{\textbf{Across the nine tasks, the direction input drives the largest gains, and waypoints make learning reliable.} Direct-DCP (Table~\ref{tab:main}) separates the two components on these tasks. On all three Pusher tasks, both direction-conditioned variants outperform both raw-input variants on both measures, so the gain comes from the actor input rather than the collection target. Collecting toward the final goal works well when the agent reaches that goal early in training, and Direct-DCP then attains the highest success, as on the Pusher tasks. When the goal is rarely reached early, as on AntPush and Humanoid U-Maze, collection toward it yields little useful experience, and Direct-DCP fails to learn in most seeds on both tasks. Waypoints provide reachable targets in this phase: every DCP seed reaches the goal on all nine tasks, compared with 38 of 45 Direct-DCP seeds, and DCP's mean success and time near goal stay at or above 59\% of the better variant's on every task (Appendix~\ref{app:direct-dcp}).}

\rev{\textbf{Waypoints also teach the agent to stay at the goal.} Each waypoint is held for 25 steps, so collection includes many stretches in which the agent reaches a target and remains there, and relabeling turns these into examples of holding a goal. On PusherHard and PusherHard Far, DCP accordingly spends more time near the goal than Direct-DCP despite lower success, remaining in the goal region for 66\% and 61\% of the steps after first reaching it, versus 39\% and 36\%.}

\rev{\textbf{The gain is not explained by the larger actor.} A raw-goal actor widened to DCP's parameter count reaches 4.7\% success on PusherEasy, close to CRL's 4.4\% and far below DCP's 55.4\% (Appendix~\ref{app:direct-dcp}).} We next examine the learned geometry and the actor's use of direction, including why the direction helps in pushing (Section~\ref{sec:interventions}).

\subsection{Does the representation capture obstacle-respecting routes?}
\label{sec:analysis}
Our motivation relies on the representation learning something useful about the routes an agent can take. We examine this directly in a planar point-mass environment where we can compute shortest paths around the walls. States include position and velocity, and actions are continuous forces. Appendix~\ref{app:geometry} describes all controlled environments and shows their layouts. A weighted occupancy graph gives exact shortest paths on a fixed spatial discretization, with diagonal edges allowed only when collision-free. We evaluate 100,128 directed state pairs and separately inspect 2,336 pairs where straight-to-goal motion conflicts with a graph-optimal first move. The graph provides a spatial reference for evaluating the learned representations.

In a five-seed representation study, DCP's goal-embedding distances correlate more strongly with U-maze graph distances than CRL's (Table~\ref{tab:geometry}). Figure~\ref{fig:maze-fields} compares the learned fields with graph distance for one seed; Appendix~\ref{app:visualizations} shows the individual-seed results and AntMaze examples. On the pairs where moving directly toward the goal conflicts with the shortest route, the negative spatial gradient of the learned distance points toward a graph-optimal first move on 70.9\% of pairs for DCP and 64.7\% for CRL, averaged across seeds. DCP's open-room graph and Euclidean correlations differ by less than 0.001 in mean, while its U-maze graph correlation exceeds its Euclidean correlation by 0.065. These measurements show that the representation captures the obstacle's effect on route length.
Each method learns its representation from its own trajectories. We also find that better distance ordering alone is insufficient to explain control performance: mean U-maze success over training is similar for DCP and CRL (0.847 and 0.846), and CRL performs better in the open room. We therefore next examine the actor's response to the learned direction.

\begin{figure}[t]
\centering
\includegraphics[width=\linewidth]{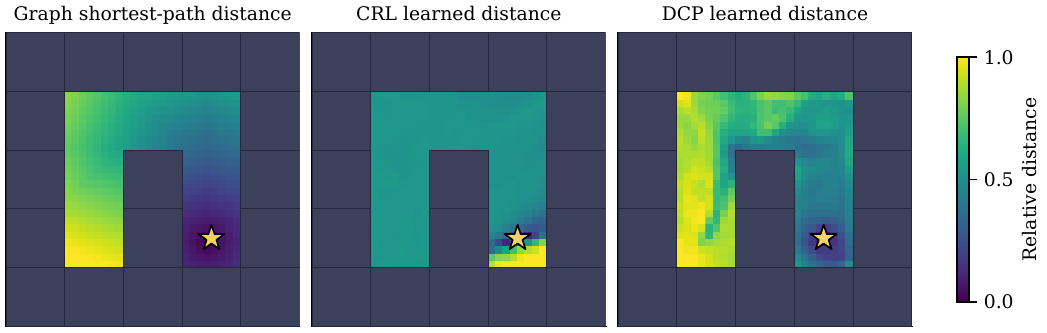}
\caption{Distance to the starred goal in the controlled U-maze: graph shortest-path distance (left), and learned representation distances for CRL and DCP after 10M steps (middle, right; seed 42). Each panel is scaled by its own 98th-percentile distance. The color scale shows relative distance within each panel.}
\label{fig:maze-fields}
\end{figure}

\begin{table}[t]
\centering
\caption{Representation geometry after 10M online steps. Means $\pm$ sample standard deviations over five independently trained encoders. Distances are computed from the goal encoder; local descent follows the negative spatial gradient of the learned distance.}
\label{tab:geometry}
\begin{tabular}{lrr}
\toprule
U-maze measure & CRL & DCP \\
\midrule
Spearman correlation with graph distance & $0.471\pm0.031$ & $0.756\pm0.041$ \\
Spearman correlation with Euclidean distance & $0.456\pm0.033$ & $0.690\pm0.034$ \\
Graph-correct local descent (\%) & $64.7\pm10.1$ & $70.9\pm4.5$ \\
\bottomrule
\end{tabular}
\end{table}

\subsection{Does the actor use the supplied direction?}
\label{sec:interventions}
If the actor uses the supplied direction to choose its actions, changing that direction should change how it reaches the goal. We test this with a fixed trained policy, keeping the encoder, distance input, initial conditions, and evaluation procedure unchanged. We first use a differential-drive U-maze with persistent velocity, forward thrust, and yaw torque; the robot cannot command arbitrary lateral force. Across three policies, 42 directed start--goal pairs, and four initial headings, the learned direction yields 45.8\% success. A reversed direction reduces success to 0.2\%, and wrong-goal, orthogonal, or zero directions also substantially reduce success (Figure~\ref{fig:interventions}). \textbf{These interventions show that the learned policies use the supplied direction to reach the goal.}

\begin{figure}[t]
\centering
\includegraphics[width=\linewidth]{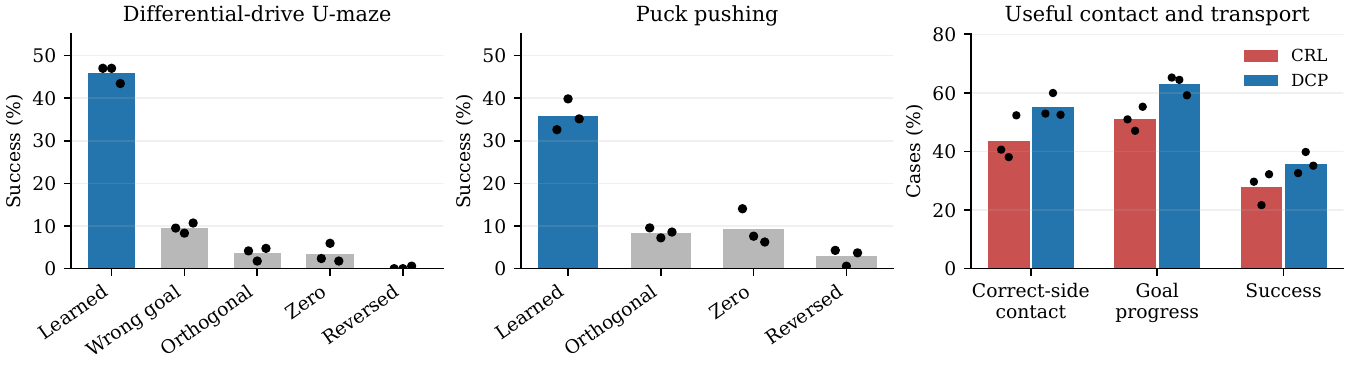}
\caption{Changing only the direction supplied to a trained actor reduces success in navigation and puck pushing (left, middle). In matched puck-pushing cases, DCP also makes more useful contacts and goal progress than CRL (right). Bars show means and dots show the three training seeds. The puck task uses 128 quartets per seed, varying the robot's side of approach while holding puck and goal fixed within each quartet.}
\label{fig:interventions}
\end{figure}

Manipulation adds another difficulty: the robot has to make the right contact before it can move the object toward the goal. We test this in a puck-pushing task, where the goal encoder receives puck coordinates and the actor receives the full robot and puck state. DCP succeeds on 35.9\% of matched cases versus 27.9\% for CRL, and makes contact from the correct side on 55.1\% versus 43.7\%. Reversing its direction reduces success to 2.9\%, even though contact remains frequent. Direction therefore helps the actor turn contact into goal-directed object motion. The full-state input supplies the robot--puck configuration needed to choose an approach.

\subsection{When does a waypoint preserve the final-goal action?}
\label{sec:substitution}
The policy learns from achieved targets during training and receives the final goal at deployment. For this change to work, a useful intermediate target should lead to an action consistent with the final goal. We test this on three trained point-mass DCP policies, evaluating 20,000 start--goal pairs per policy with two alternative targets: an exact midpoint on a graph-shortest path, and an off-path control matched in distance from the source. Holding the actor and physical state fixed, we compare each target's action with the final-goal action. On-path substitution yields mean action cosine 0.789, versus $-0.072$ for the off-path control; positive alignment with a graph-optimal first move occurs on 95.2\% versus 41.9\% of cases. The on-path target performs better on these measures across all seeds. Here, the graph provides the midpoint, and substituting it changes both the direction and the distance. We find that the learned policy largely preserves its action when the waypoint lies on the route to the goal. The next question is whether our scorer actually finds such waypoints. Appendix~\ref{app:geometry} relates latent triangle slack to direction agreement and includes the radius difference in the actor-output bound.

\subsection{How well does the scorer select useful waypoints?}
\label{sec:limits}
A useful waypoint has to be present among the candidates and then chosen by the scorer. We examine both steps using five trained U-maze encoders and 20,000 matched start--goal pairs and candidate sets per encoder, sampled from the occupancy graph. An exact forward on-path candidate is available in 77.7\% of draws. Uniform selection chooses one in 8.3\%, while the inner-product scorer selects one in 12.9--18.8\%. The scorer also reduces median normalized path slack from 0.761 to 0.197--0.441 and changes average normalized progress from negative to positive. We therefore see a clear improvement over random selection, along with a substantial gap between having a useful candidate and choosing it. Because candidates are sampled from the graph, this experiment tests the scoring rule separately from the states collected during training.

A separate three-seed nested-candidate test increases availability from 77.9\% with 32 draws to 94.3\% with 128 draws, while useful selection rises by only 2.4--3.9 percentage points. This gap identifies ranking as a limitation even when useful candidates are available. The pool introduces a further constraint. If a fraction $p$ of its states is useful, $N$ independent draws contain at least one useful state with probability $1-(1-p)^N$. When $p=0$, increasing the number of draws cannot help. Alternative ranking rules \rev{further} improve the fixed-encoder test (Appendix~\ref{app:geometry}).

\subsection{Can a learned candidate source improve exploration?}
The finite pool limits selection to states the agent has already visited. To address this, we replace pool samples with candidates from an online four-component mixture model trained on trajectory midpoints \citep{jurgenson2020sub}. We keep the actor, critic, scorer, candidate count, and 10M-step budget fixed. The generator conditions on the current achieved-goal coordinates and the final goal. We also test whether restricting generated proposals to observed states helps: a third variant projects each proposal to its nearest stored state before scoring. The nine-task comparison uses the default visited-state pool.

\begin{table}[t]
\centering
\caption{Changing the candidate source in two controlled mazes. Mean success over training, averaged across five seeds. $\pm$ denotes sample standard deviation across seeds. Projection maps each proposal to the nearest stored state.}
\label{tab:generator}
\begin{tabular}{lrrr}
\toprule
Task & Visited pool & Generator & Generator + projection \\
\midrule
Two-route maze & $0.457\pm0.301$ & $0.613\pm0.202$ & $0.738\pm0.128$ \\
Rare-gateway maze & $0.271\pm0.150$ & $0.520\pm0.113$ & $0.204\pm0.291$ \\
\bottomrule
\end{tabular}
\end{table}

The generator increases mean success over training in both tasks (Table~\ref{tab:generator}), exceeding the pool in four of five paired seeds in the two-route maze and all five in the rare-gateway maze. Projection improves the mean in the two-route setting but removes the benefit at the rare gateway, where it can move proposals to the wrong side of the passage. The generator extends candidate coverage beyond the visited-state pool. Appendix~\ref{app:generator} describes its training and the projection step.

\subsection{Does seeding the pool with successful states help?}
A more direct way to improve candidate availability is to add states from successful trajectories. We use this intervention to test whether supplying useful states is enough to improve control, while keeping replay unchanged. Each task uses three target seeds and matched libraries of successful-path states and states preceding meaningful object motion. On PusherHard, reserving 128 of the 512 slots for successful-path states increases mean success over training from 32.39\% with the matched pre-motion library to 38.17\%, winning all three paired seeds. However, adding 128 states while preserving all 512 online slots and their expected exposure gives no consistent benefit: successful-path seeding scores 36.26\%, compared with 38.09\% for the matched control. On AntSoccer, the scorer selects successful-path states on 30.91\% of refreshes, versus 4.59\% for pre-motion states, and waypoint-to-goal direction agreement improves in every seed. \textbf{Ball motion and success do not improve.} Analysis of completed training runs also finds sparse ball motion and little change from the initial representation's distance ordering (Appendix~\ref{app:diagnostics}). In this case, the scorer recognizes the added states, but the policy still fails to turn them into useful ball motion.

%% file: sections/5_conclusion.tex
\section{Conclusion}
\label{sec:conclusion}
We asked whether a CRL policy could make better use of the geometry learned by its critic. DCP supplies the actor with the direction and distance to a target in representation space, using visited waypoints during collection and the final goal at deployment. Across nine tasks, DCP improves final mean success and time near goal over CRL on seven tasks each. \rev{Our ablations show that conditioning on direction and distance accounts for the largest gains, while waypoint collection makes learning reliable when the goal is rarely reached early in training and helps the agent stay at the goal once it is reached. Conditioning directly on the final goal can reach higher success, but it can fail to learn when the goal is hard to reach early. Full DCP is thus a trade-off, but a better bet when one does not want to gamble on whether direct goal conditioning would suffice.}

Our analyses show that the representation captures route structure and that the actor uses the supplied direction in navigation and manipulation. \rev{Replacing the visited-state pool with learned candidate generation also improves goal reaching in two controlled mazes.}

Our experiments use simulated state observations. Further work could extend DCP to visual observations\rev{,} combine it with deeper or action-chunked CRL\rev{, use other learned goal representations in place of the critic's encoder, and study the geometry of CRL's representation space theoretically}. \rev{We hope this work encourages others to use learned representations as policy inputs, even while these representations are still being trained.}

%% file: statements.tex
\section*{Reproducibility statement}
The DCP implementation and reproduction scripts are available at \expandafter\url\expandafter{\codeurl}.
Appendix~\ref{app:implementation} specifies the objectives, policy inputs, and collection procedure. Appendix~\ref{app:tasks} gives task definitions, budgets, and experimental settings. The final task comparisons use five seeds; additional experiments state their seed counts separately. Figure~\ref{fig:curves} and Appendix~\ref{app:curves} provide learning curves and checkpoint values. Appendix~\ref{app:geometry} gives the proof relating latent triangle slack to direction agreement.

\section*{Ethics statement}
This work studies goal-conditioned learning in simulated environments with state observations. No human-subject data are used. Physical deployment requires safety constraints and hardware validation, particularly for contact-rich tasks.

\section*{AI use statement}
Generative AI tools were used to assist with research code, analysis of experimental results, literature retrieval, and drafting and revising the manuscript. They were also used in discussing mathematical claims and proofs. The research ideas originated with the authors. The authors are responsible for the accuracy of the final text, mathematical statements, citations, code, and reported results.

%% file: appendix.tex
\section{Implementation and Objectives}
\label{app:implementation}
\subsection{Default algorithm}
Algorithm~\ref{alg:dcp} gives the DCP procedure used in the nine-task comparison. We use the visited-state pool and inner-product scorer throughout that comparison; the generator, demonstration-state libraries, and alternative scoring rules are studied in separate experiments. The policy outputs a tanh-squashed Gaussian distribution; evaluation uses the deterministic policy. The actor input concatenates the full state, a 64-dimensional direction, and one radius.

\begin{algorithm}[h]
\caption{Direction-Conditioned Policies}
\label{alg:dcp}
\begin{algorithmic}[1]
\STATE Initialize actor $\pi_\eta$, encoders $(\phi_\theta,\psi_\theta)$, replay, and state pool.
\WHILE{training budget remains}
  \STATE Observe full state $s$ and requested final goal $g$.
  \IF{waypoint refresh is due}
    \STATE Sample 32 candidate indices with replacement from the 512-state pool.
    \STATE Select $z^*$ maximizing $\psi_\theta(q(z))^\top\psi_\theta(g)$.
    \STATE Set collection target $y\leftarrow q(z^*)$.
  \ENDIF
  \STATE Recompute $(d,r)$ from $(s,y)$ using Equation~\ref{eq:interface}.
  \STATE Sample and execute $a\sim\pi_\eta(\cdot\mid s,d,r)$; store the transition.
  \STATE Refresh the pool from the collected trajectory batch at the collection boundary.
  \STATE Sample replay sequences and relabel transitions with future achieved goals.
  \STATE Update encoders with backward InfoNCE and the baseline regularizer.
  \STATE Update actor and entropy coefficient, holding encoder parameters fixed.
\ENDWHILE
\STATE At deployment: use $y=g$ at each step, without the pool or scorer.
\end{algorithmic}
\end{algorithm}

\paragraph{Numerical details.}
The implemented energy is $-\sqrt{\norm{\phi(s,a)-\psi(g)}_2^2+10^{-6}}$. The goal displacement is divided by $r+10^{-8}$; a zero displacement therefore produces a zero direction. In addition to Equation~\ref{eq:infonce}, the critic uses
\begin{equation}
 \mathcal L_{\mathrm{reg}}=
 0.1\frac{1}{B}\sum_i\left[\log\sum_j\exp(\ell_{ij}+10^{-6})\right]^2.
\end{equation}
We differentiate the actor loss through the critic's action input while holding both encoders fixed. The temperature update uses the baseline entropy target.

\paragraph{Future-goal relabeling.}
Within a sampled sequence, a future index $j>i$ receives weight proportional to $\gamma^{j-i}$, restricted to the same trajectory identifier. A small diagonal weight handles positions with no eligible future sample. The selected achieved-goal coordinates provide the critic positive and actor target.

\paragraph{Pool maintenance.}
The default pool is refreshed from evenly spaced indices of the flattened latest collection batch. Its initial entries are zero-initialized before collected observations replace them. Candidate indices are shared across parallel environments, but candidates are ranked separately for each requested goal. The waypoint is refreshed every 25 steps according to the implementation's environment-step schedule. Encoder evaluations of the current achieved goal and selected target produce an updated direction on every action.

\subsection{Training settings}
Table~\ref{tab:hyperparameters} lists the shared optimization settings and DCP's waypoint settings. Humanoid U-Maze uses fewer parallel environments and a smaller minibatch; comparisons within that task retain the same settings across methods.
\begin{table}[h]
\centering
\caption{Main settings. Humanoid U-Maze uses the stated smaller parallelism and batch size. Replay sizes retain the implementation's buffer configuration units.}
\label{tab:hyperparameters}
\begin{tabular}{lr}
\toprule
Setting & Value \\
\midrule
Actor / critic hidden width & 256 \\
Hidden layers & 2 \\
Representation dimension & 64 \\
Actor, critic, and temperature learning rates & $3\times10^{-4}$ \\
Discount & 0.99 \\
Minibatch size & 256 (Humanoid: 128) \\
Parallel environments & 512 (Humanoid: 128) \\
Episode length & 1,000 \\
Collection unroll length & 62 \\
Minimum / maximum replay size & 1,000 / 10,000 \\
Contrastive loss / energy & Backward InfoNCE / negative L2 norm \\
Log-sum-exp penalty coefficient & 0.1 \\
Pool size / candidates / refresh steps & 512 / 32 / 25 \\
\bottomrule
\end{tabular}
\end{table}

We train DCP, CRL, and SSGC with the same within-task simulation settings and backward InfoNCE configuration. All baseline results come from these matched runs.

\subsection{Parameter counts}
\label{app:parameter-counts}
We count saved actor and critic parameters for all five seeds of each benchmark task. Counts agree across seeds. DCP replaces the two or three raw goal coordinates with 64 direction components and one distance. With first-layer width 256, this adds $256(65-\dim(g))$ weights: 16,128 for two-dimensional goals and 15,872 for three-dimensional goals. All remaining actor layers and both critic encoders have identical parameter counts. The totals in Table~\ref{tab:parameter-counts} include the actor and both encoders; the learned entropy coefficient adds one scalar to each method.

\begin{table}[ht]
\centering\small
\caption{Saved parameter counts for the nine benchmark tasks. Grouped tasks have identical architectures and counts. Percentage increases are relative to CRL.}
\label{tab:parameter-counts}
\begin{tabular}{lrrrr}
\toprule
Tasks & CRL actor & DCP actor & Actor increase & Total increase \\
\midrule
AntMaze Big / Hardest / U-Maze & 78,096 & 94,224 & 20.65\% & 6.37\% \\
Humanoid U-Maze & 144,162 & 160,034 & 11.01\% & 4.15\% \\
AntPush / AntSoccer & 78,608 & 94,736 & 20.52\% & 6.35\% \\
PusherEasy / Hard / Hard Far & 75,534 & 91,406 & 21.01\% & 6.40\% \\
\bottomrule
\end{tabular}
\end{table}

The controlled point-mass actor has a smaller physical-state input and fewer action outputs, so the added input weights account for a larger relative increase (23.5\%). Across the benchmark, that increase is 11--21\% for the actor and 4.15--6.40\% for the complete actor--critic. The capacity difference is thus confined to the input layer. \rev{Appendix~\ref{app:direct-dcp} reports a raw-goal actor with matched parameters on PusherEasy.}

\paragraph{Computational cost.}
We measure training throughput over three seeds using the median steps per second in the second half of each run. This reduces the effect of compilation and logging on the comparison. DCP runs at 95.3\% of CRL's throughput at the median across tasks, with task ratios ranging from 92.4\% to 97.6\%. A 512-entry float32 state pool occupies 0.523 MiB at the largest state size (268). The pool and scoring are removed at deployment; goal-encoder evaluation is still required.

\section{Tasks and Evaluation Protocols}
\label{app:tasks}
Table~\ref{tab:tasks} specifies each task's goal coordinates, success radius, and training budget. These determine what counts as reaching the goal: AntPush measures the robot's position, whereas AntSoccer measures the ball's position. All three methods receive the same interaction budget for a given task.
\begin{table}[h]
\centering
\caption{The nine evaluation tasks. Training steps are matched across methods within each task.}
\label{tab:tasks}
\begin{tabular}{llll}
\toprule
Task & Goal coordinates & Radius & Training steps \\
\midrule
AntMaze Big & Ant position & 0.5 & 50M \\
AntMaze Hardest & Ant position & 0.5 & 50M \\
Ant U-Maze & Ant position & 0.5 & 50M \\
Humanoid U-Maze & Humanoid position & 0.5 & 100M \\
AntPush & Ant position (movable obstacle) & 0.5 & 100M \\
AntSoccer & Ball XY & 0.5 & 50M \\
PusherEasy & Cylinder position & 0.1 & 50M \\
PusherHard & Cylinder position & 0.1 & 100M \\
PusherHard Far & Cylinder position & 0.1 & 50M \\
\bottomrule
\end{tabular}
\end{table}

\paragraph{Navigation.}
AntMaze Big and AntMaze Hardest place an eight-actuator quadruped in different corridor layouts. Ant U-Maze uses the same robot in a U-shaped corridor, where reaching the opposite arm requires moving around the central wall. Humanoid U-Maze uses a 17-actuator humanoid in that layout. Actions control the robot's joints. The actor observes the physical state, including positions and velocities. Ant navigation uses horizontal root position as the achieved goal; Humanoid uses three-dimensional root position, with target height fixed by the environment. Success requires a goal distance below 0.5.

\paragraph{Object interaction.}
AntPush requires the Ant to reach a target position in a passage obstructed by a movable block. Success is measured from the Ant's position. AntSoccer uses the same type of quadruped to move a ball; the achieved goal is the ball's horizontal position, and success requires ball-to-target distance below 0.5. AntSoccer is called \texttt{ant\_ball} in the environment implementation; these are two names for the same task.

\paragraph{Robot-arm manipulation.}
The Pusher tasks use a seven-joint arm to push a cylinder. The achieved goal is the cylinder's three-dimensional position, and success requires distance below 0.1. The variants differ only in their target distribution. Let $(u,v)$ denote the two planar goal-joint offsets, in their stored order. PusherEasy samples $(u,v)$ uniformly from $[-0.55,-0.25]\times[-0.2,0.2]$; PusherHard uses $[-0.65,0.35]\times[-0.55,0.45]$; PusherHard Far uses $[0,0.35]\times[-0.55,0.45]$. The first joint translates along world $y$ and the second along world $x$, relative to the goal body's reference position. All variants initially sample cylinder offsets from $[-0.3,0]\times[-0.2,0.2]$, with the same adjustment when the cylinder and goal are too close. Figure~\ref{fig:pusher-views} provides additional views of target placement.

For an episode of length $T=1000$, let $I_t$ indicate membership in the goal region. Success is $\mathbf{1}\{\sum_t I_t>0\}$, averaged over episodes and expressed as a percentage. Time near goal is $\sum_t I_t$, averaged over episodes. Where we report mean success over training in controlled studies, the statistic is the arithmetic mean of the logged evaluation success rates across the fixed evaluation schedule, first computed per seed. Comparisons within each study use the same schedule.

\begin{table}[h]
\centering\small
\caption{Number of seeds and evaluations in each experiment. The geometry experiment and four-way comparison use separately trained policies.}
\label{tab:study-protocols}
\begin{tabular}{p{0.42\linewidth}p{0.48\linewidth}}
\toprule
Experiment & Training and evaluation \\
\midrule
Nine-task comparison & Five seeds per method and task; 50M or 100M steps. \\
Four-way collection/input comparison & Five seeds, 10M steps, 20 evaluations, separately in each layout. \\
Representation geometry & Five encoders per method and layout, trained for 10M steps. \\
Candidate-source comparison & Five seeds per candidate source and maze, 10M steps. \\
Waypoint substitution & Three policies; 20,000 paired probes per policy. \\
Differential-drive intervention & Three policies; 42 directed pairs and four headings per policy. \\
Puck-pushing intervention & Three policies; 128 matched quartets per policy. \\
Demonstration-state seeding & Three target seeds; successful states added to the candidate pool. \\
Frozen-CRL transfer & One behavioral seed; 50M source pretraining plus 25M adaptation. \\
\bottomrule
\end{tabular}
\end{table}

Table~\ref{tab:study-protocols} distinguishes the main five-seed comparison from the smaller intervention and transfer studies. Each seed corresponds to a separately trained policy or encoder; multiple probes of one policy share that training run. The five-seed studies use seeds 1, 42, 123, 256, and 1234; three-seed interventions use 1, 42, and 123. Uncertainty is computed across independent training seeds. The benchmark and four-way comparison use Student-$t$ intervals with four degrees of freedom: $\bar x\pm2.77645\,s/\sqrt{5}$, where $s$ is sample standard deviation across seeds. Geometry and candidate-source tables instead report sample standard deviations.

\section{Controlled Geometry and Selection Tests}
\label{app:geometry}
\subsection{Controlled environments}
Figure~\ref{fig:controlled-environments} shows the six controlled environments. These supplement the nine benchmark tasks; their results are reported separately. All use 1,000-step episodes. The point-mass maze studies train for 10M interactions and evaluate at 20 checkpoints. We use the same dynamics and training settings across methods within each layout.

\begin{figure}[t]
\centering
\includegraphics[width=\linewidth]{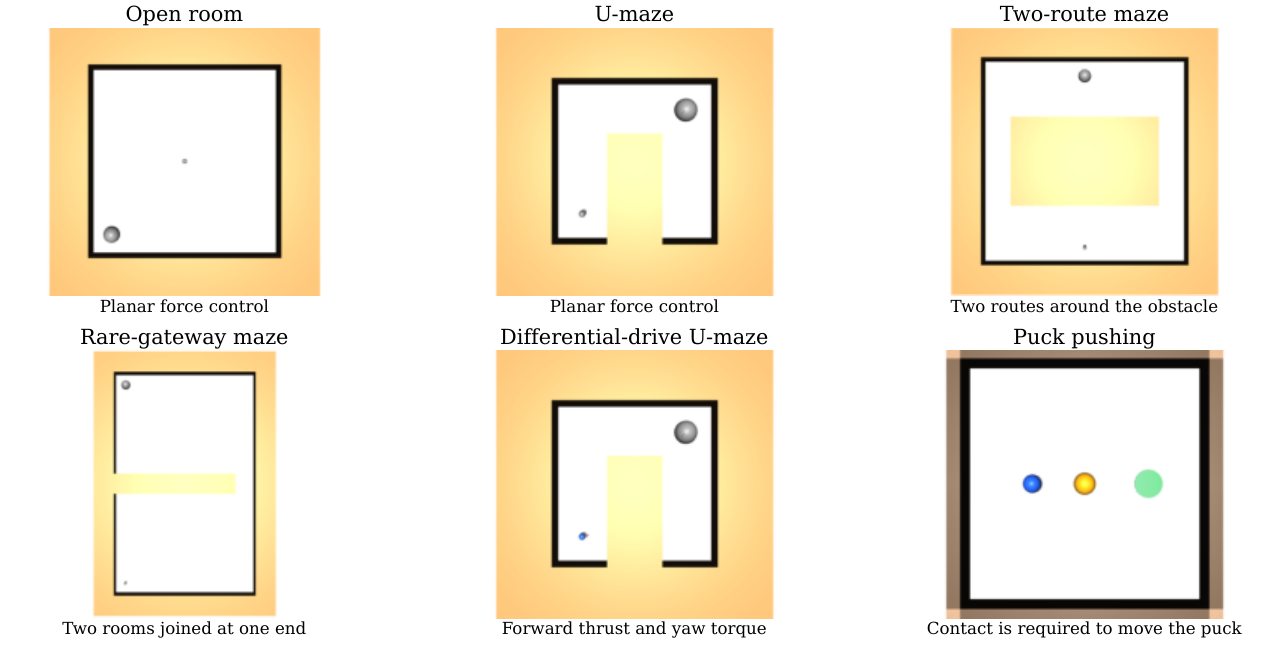}
\caption{The controlled environments, rendered from their simulation models with floor textures hidden. The four point-mass environments differ in wall layout and start/goal distributions. The differential-drive robot adds heading-dependent motion. In puck pushing, the blue agent must contact the yellow puck and move it to the green target. Images show illustrative task configurations.}
\label{fig:controlled-environments}
\end{figure}

\paragraph{Point-mass mazes.}
The state is $(x,y,v_x,v_y)$, the action is a two-dimensional force, and the goal is a position. Forces change velocity, which persists between steps. Wall cells have side length 4, and success requires position error below 0.5. The open room has boundary walls only and starts near its center. The U-maze starts in one arm, with goals sampled from allowed free cells. The two-route maze places start and goal on opposite sides of a central block, allowing a route around either end. The rare-gateway maze has two rooms separated by a wall, with a single passage near one end; the start and goal lie across the divider, away from that passage. The generator experiments use the latter two layouts.

\paragraph{Differential-drive U-maze.}
The physical state is $(x,y,\theta,v_x,v_y,\dot\theta)$; the two actions apply forward thrust along the current heading and yaw torque. The robot cannot directly command sideways force, and velocity persists. The goal is horizontal position, with success radius 0.5. For the direction test, each of three learned policies is evaluated on 42 directed start--goal pairs and four initial headings per pair, holding these cases fixed across interventions.

\paragraph{Puck pushing.}
A planar force-controlled robot pushes a separate, unactuated puck. The actor observes both positions and velocities, an eight-dimensional physical state; the goal encoder receives only puck position. Success requires puck-to-target distance below 0.55. Contact transfers the robot's motion to the puck. The matched evaluation uses 128 groups of four cases per seed. Within a group, the puck and target are fixed while the robot approaches from four different sides. Groups vary target direction and distance, absolute position, and initial robot--puck separation. This tests whether the actor combines its direction input with the physical arrangement needed for useful contact.

\subsection{Graph reference and interpretation}
The point-mass state is $(x,y,v_x,v_y)$ and the action is a two-dimensional force. The U-maze is represented by a weighted occupancy graph with cardinal and collision-free diagonal connections. For the global-distance comparison, the graph supplies 100,128 directed pairs of distinct positions. Graph distances measure spatial route lengths. Local-direction probes compare the negative spatial gradient of $\norm{\psi(q(s))-\psi(g)}$ with graph-optimal outgoing moves. The actor input is defined in Equation~\ref{eq:interface}.

In the open room, graph distance and Euclidean distance largely agree. The layouts have different initial-state and goal distributions; comparisons of scoring rules and actor inputs are performed within each layout.

\subsection{Comparing waypoint selection rules}
We want to know whether the scoring rule chooses useful candidates when they are available. We therefore sample start, goal, and candidate positions from the occupancy graph, hold the encoders fixed, and give each rule the same 20,000 draws of 32 candidates.

With graph distance $D$, a candidate is counted as forward on-path when
\begin{equation}
 D(s,z)+D(z,g)-D(s,g)=0,\qquad D(z,g)<D(s,g),
\end{equation}
up to the numerical tolerance. Relative slack divides the left-hand slack by $D(s,g)$; relative progress is $[D(s,g)-D(z,g)]/D(s,g)$. Both require distinct source and goal. The default graph uses eight subdivisions and 448 nodes. Additional resolutions test sensitivity to discretization.

We compare two alternative scoring rules on the same candidate sets without retraining the encoders. One maximizes the cosine between $\psi(z)-\psi(s)$ and $\psi(g)-\psi(s)$. The other minimizes $\norm{\psi(z)-\psi(s)}+\norm{\psi(g)-\psi(z)}$. Across five trained encoders, the default inner-product scorer selects a forward on-path candidate in 12.9--18.8\% of draws; cosine ranking achieves 24.7--27.8\%, and minimizing the sum of learned distances achieves 28.6--33.0\%. Uniform selection achieves 8.3\%. Both alternatives improve waypoint selection in this test. Whether they improve online learning remains untested: changing the selection rule would also change the experience collected and the representation learned.

\subsection{A geometric statement about the direction input}
When a waypoint and goal lie in similar directions from the current embedding, they provide similar direction inputs to the actor. We can quantify this relation using the slack in the triangle formed by their embeddings. We use exact normalization and nonzero displacements for the statement below. Let $x=\psi(q(s))$, $w=\psi(q(z))$, and $v=\psi(g)$. Define
\begin{equation}
 a=\norm{w-x},\quad b=\norm{v-w},\quad c=\norm{v-x},\quad
 h=a+b-c,\quad u_z=\frac{w-x}{a},\quad u_g=\frac{v-x}{c}.
\end{equation}
\begin{proposition}[Latent slack and direction agreement]
For $a,c>0$, the triangle slack $h\geq0$ satisfies
\begin{equation}
 \norm{u_z-u_g}^2=\frac{h(b+c-a)}{ac}\leq\frac{2h}{a}.
 \label{eq:slack}
\end{equation}
\end{proposition}
\begin{proof}
The cosine law gives $b^2=a^2+c^2-2ac\langle u_z,u_g\rangle$. Consequently,
\begin{align}
 ac\norm{u_z-u_g}^2
 &=b^2-(a-c)^2\\
 &=(a+b-c)(b+c-a)=h(b+c-a).
\end{align}
The triangle inequality $b\leq a+c$ implies $b+c-a\leq2c$, proving the bound.
\end{proof}

This identity describes the geometry of the Euclidean embedding. If a deterministic actor is separately Lipschitz in direction and radius, with constants $L_d,L_r$ at a fixed state, then
\begin{equation}
 \norm{\pi(s,u_z,a)-\pi(s,u_g,c)}
 \leq L_d\sqrt{2h/a}+L_r|a-c|.
 \label{eq:full-input}
\end{equation}
The distance input matters here. Even when $h=0$, the waypoint and goal can be at different distances from the current embedding. For example, $x=0$, $w=1$, and $v=2$ give identical directions and different radii. Equation~\ref{eq:full-input} accounts for both changes.

The distance $\norm{\psi(q(s))-\psi(g)}$ is symmetric in its two goal-coordinate arguments and may vanish between distinct coordinates. It is therefore a pseudometric on those coordinates. Relating it to optimal-control cost requires additional assumptions, as in metric constructions for contrastive successor features and quasimetric value learning~\citep{myers2024distances,wang2023qrl}.

\section{Candidate Generation and Demonstration Seeding}
\label{app:generator}
\subsection{Online candidate generator}
The generator is a conditional mixture-density network with four diagonal-Gaussian components and two hidden layers of width 128. Its inputs are the current achieved-goal coordinates and a desired endpoint. Training samples an ordered triple of indices $i<k<j$ from a replay sequence, with $k=\lfloor(i+j)/2\rfloor$, and fits the midpoint coordinates by negative log likelihood. Triples crossing trajectory identifiers are masked out. The learning rate is $3\times10^{-4}$ and the generator warm-up is 1M interaction steps. Component scales are constrained between 0.05 and 2, with initial scale and offset scale 0.5. Candidate count stays at 32.

The unprojected variant uses sampled goal-coordinate proposals as the waypoint targets. The projected variant maps proposals to the nearest pool state's goal coordinates, then applies the same scorer. Projection returns a stored target but can move a proposal across an obstacle. All three variants retain the same deployment policy input. The default DCP configuration uses the visited-state pool.

\subsection{Demonstration-state libraries}
We construct each task's libraries from nine successful DCP episodes, three per source seed. The successful-path library contains 128 states chosen around object displacement. For the control, we take the same number of states from those episodes before meaningful object motion. This lets us test the value of the states supplied to the pool while keeping replay and contrastive training unchanged.

In the reserved-slot comparison, each library occupies 128 of the 512 slots, leaving 384 online slots. The augmented comparison instead preserves all 512 online slots and adds 128 library slots. It draws 40 candidates so that the expected number of online draws remains 32, approximately preserving each online state's exposure. The reserved-slot experiment tests the value of library content under constrained capacity. The augmented experiment tests the benefit of adding those states to the full online pool.

\begin{table}[h]
\centering
\caption{PusherHard demonstration-state experiments, three target seeds. Entries are mean evaluation success over training (\%). Paired wins compare successful-path against pre-motion seeding. DCP without seeding scores 35.76\% in these experiments.}
\label{tab:demo-seeding}
\begin{tabular}{lrrr}
\toprule
Pool intervention & Pre-motion & Successful path & Paired wins \\
\midrule
Reserve 128 of 512 slots & 32.39 & 38.17 & 3/3 \\
Add 128 to 512 online slots & 38.09 & 36.26 & 0/3 \\
\bottomrule
\end{tabular}
\end{table}

Table~\ref{tab:demo-seeding} shows that useful demonstration states help when they replace part of a capacity-limited pool: successful-path seeding exceeds the matched pre-motion control in all three seeds. Adding these states to the complete online pool gives no corresponding benefit. The comparison therefore depends on both the content of the added states and which online candidates they replace.

For AntSoccer, successful-path seeding improves mean waypoint/final-goal direction cosine from 0.596 to 0.655 and latent slack from 1.096 to 1.000, without improving ball motion or success. \rev{Direct-DCP resolves this task, matching CRL within its confidence interval (Appendix~\ref{app:direct-dcp}).}

\section{Additional Analyses}
\label{app:diagnostics}
\subsection{Learning early in training and maintaining success}
The direction input is learned online, so we also examine whether an initially inaccurate representation delays learning and whether success is maintained once it begins. In a three-seed analysis of the nine tasks, DCP crosses 2\% success earlier than CRL on seven tasks and ties on two. At 5\% and 10\%, it is earlier on six, tied on one, and later on two. Its success over the first 25M steps is higher on seven tasks, tied on AntPush, and lower on AntSoccer. These comparisons concern early learning, separately from the final five-seed results.

After first crossing each threshold, we measure the fraction of the next ten evaluations that falls below it. The median fraction at 5\% and 10\% is 0\% for DCP, compared with 10\% and 20\% for CRL. This measures persistence after initial success.

On AntPush, mean DCP success decreases from 44.8\% at 75M steps to 42.7\% at 100M steps in the five-seed results. The learning curves show the full course of training, including this late decline.

\subsection{AntSoccer representation and activity}
To understand what happens on AntSoccer, we examine how much the representation changes and how often the agent moves the ball. Across five seeds, we compare distance ordering on fixed coordinate pairs with its ordering at initialization. On AntSoccer, the mean Spearman-correlation change is only $-0.050$ at 25M and $-0.021$ at the final checkpoint; the final changes are larger in magnitude on AntPush and the Pusher tasks. Only 3.8\% of retained consecutive AntSoccer transitions move the ball by more than $10^{-3}$. AntSoccer combines sparse object movement with little change in distance ordering.

We measured these signals in completed runs. They identify AntSoccer as unusual among the evaluated tasks. Tracking fixed coordinate probes and replay activity would also make them available during training; predicting individual-run outcomes requires further testing.

\subsection{Frozen CRL representations}
If direction conditioning can use geometry learned by CRL, it should also be possible to train a new actor on a frozen CRL representation. We test this directly. A single-seed AntMaze Big experiment starts from a CRL checkpoint trained for 50M steps and trains a new DCP actor for 25M additional interactions. One variant freezes only the goal encoder and adapts the state--action encoder; another freezes both encoders. A fresh joint DCP receives 25M steps without pretraining.

\begin{table}[h]
\centering
\caption{Frozen-representation transfer on AntMaze Big, seed 42. The transferred variants use 50M CRL pretraining steps and 25M actor-training steps; fresh DCP uses 25M steps in total.}
\label{tab:frozen-transfer}
\begin{tabular}{lrrr}
\toprule
Variant & Pretraining & Final success (\%) & Mean training success (\%) \\
\midrule
Fresh joint DCP & 0 & 25.4 & 16.72 \\
Frozen CRL goal encoder & 50M & 22.3 & 11.60 \\
Both CRL encoders frozen & 50M & 23.0 & 12.46 \\
\bottomrule
\end{tabular}
\end{table}

Table~\ref{tab:frozen-transfer} shows that a new direction-conditioned actor can use a CRL representation even when both encoders are frozen, reaching 23.0\% final success. Fresh DCP reaches 25.4\% with a smaller total interaction budget. This establishes reuse in the tested run, while the unequal pretraining budgets preclude a sample-efficiency comparison.

The frozen encoder parameters remain unchanged at all 26 checkpoints examined during adaptation. Separate implementation checks confirm that 15 CRL checkpoints across AntMaze Big, PusherHard, and AntSoccer can be loaded and kept frozen while the actor is updated. The AntMaze Big experiment demonstrates behavioral transfer for one seed. Its generality across seeds and tasks remains to be tested.

\section{Additional Visualizations}
\label{app:visualizations}
\paragraph{Variation across training seeds.}
Figure~\ref{fig:geometry-seeds} expands Table~\ref{tab:geometry} into individual-seed comparisons. Each line connects encoders trained with the same seed. The first two panels measure how well learned distances order graph shortest-path distances, over all state pairs and over the longest-distance third. DCP improves both measures in every seed. The third panel measures local descent: whether decreasing learned distance initially points toward a graph-optimal move. Its larger variation shows that improved global distance ordering can coexist with local directional errors.
\begin{figure}[ht]
\centering
\includegraphics[width=\linewidth]{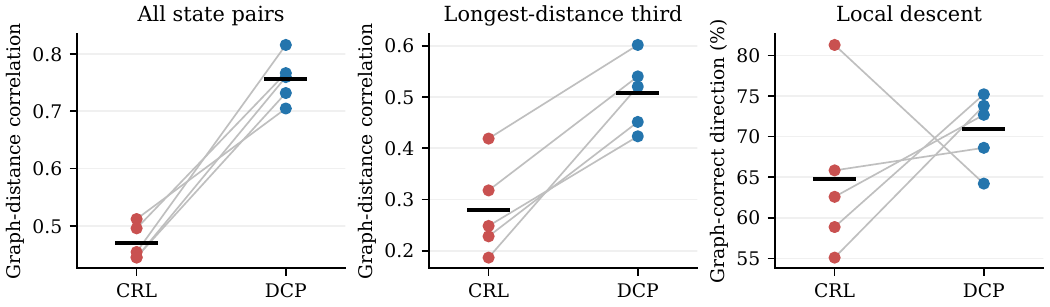}
\caption{U-maze representation measurements for all five seeds after 10M steps. Connected dots compare CRL and DCP with the same seed; black horizontal marks show means. DCP has higher graph-distance correlation in every seed, both over all pairs and over the longest-distance third. The local-descent comparison shows more variation across seeds.}
\label{fig:geometry-seeds}
\end{figure}

\paragraph{Distance fields in AntMaze Big.}
Figures~\ref{fig:antmaze-fields} and~\ref{fig:antmaze-fields-seed1} place learned distances alongside executed trajectories. For a fixed goal, we evaluate the goal encoder over spatial coordinates and plot the resulting distance to the goal embedding. The first example shows successful trajectories for both methods; the second shows DCP reaching the goal while CRL stalls. Comparing the contours with the maze walls helps reveal whether nearby physical positions are separated by learned distance when a wall forces a detour. Colors are scaled within each panel, so the comparison concerns spatial structure. These examples illustrate behaviors that contribute to the aggregate results; the five-seed learning curves report their broader variability.

\begin{figure}[ht]
\centering
\includegraphics[width=0.9\linewidth]{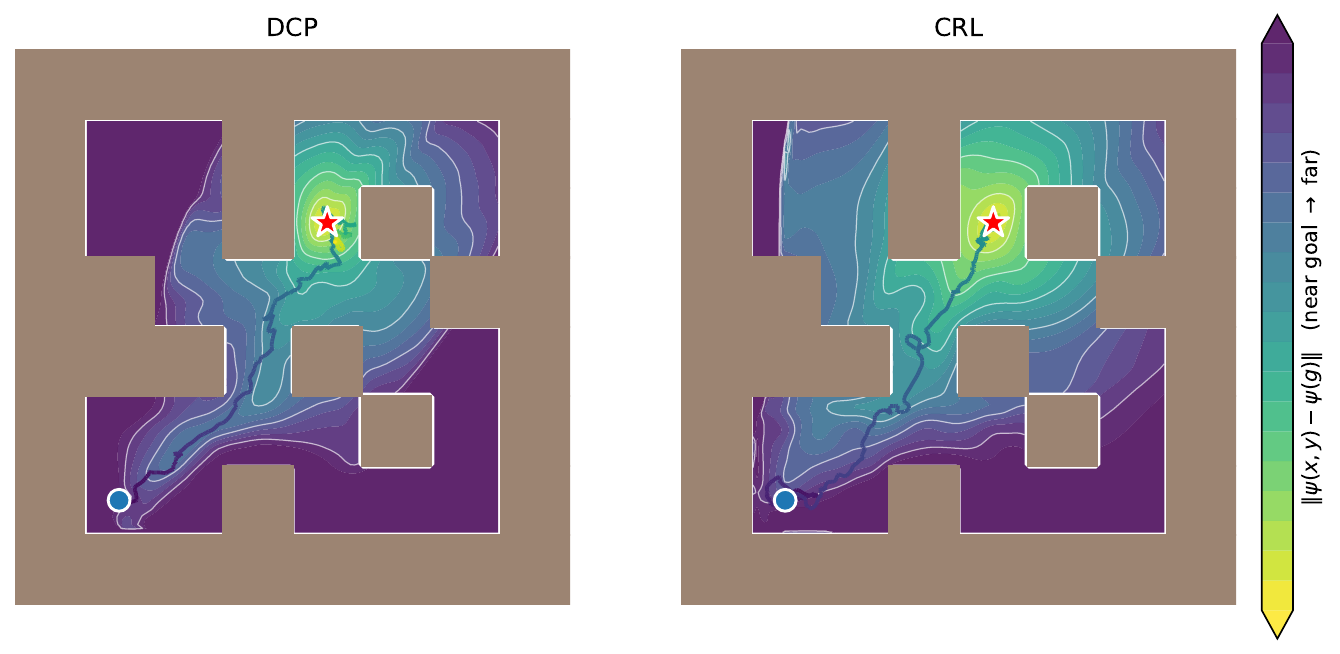}
\caption{AntMaze Big distance fields and example trajectories for a seed in which both DCP and CRL reach the goal. Contours show $\|\psi(x,y)-\psi(g)\|$; the blue circle and red star mark start and goal. Both fields reflect the corridor layout. Colors are scaled separately for each method.}
\label{fig:antmaze-fields}
\end{figure}

\begin{figure}[ht]
\centering
\includegraphics[width=0.9\linewidth]{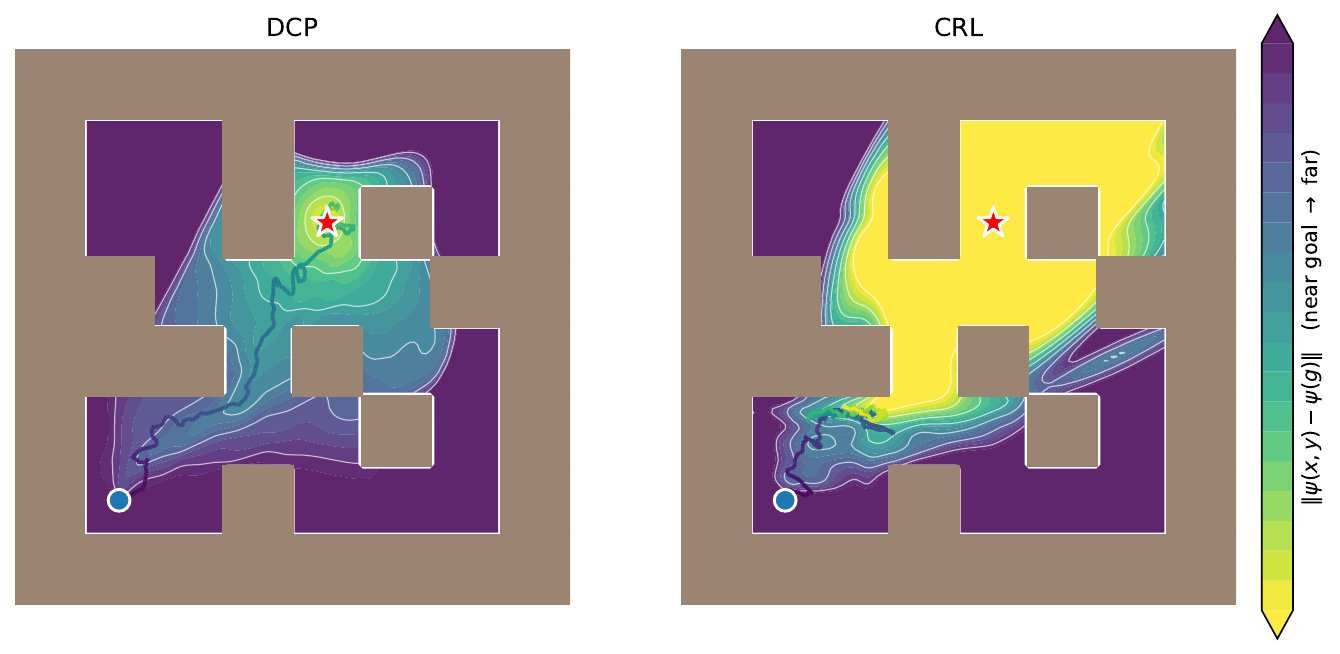}
\caption{A second AntMaze Big example (seed 1), in which DCP reaches the goal and CRL stalls. The color range in each panel is clipped to the 1st--99th percentiles of distances along that method's trajectory. The field and trajectory illustrate one outcome; the main comparison averages all five seeds.}
\label{fig:antmaze-fields-seed1}
\end{figure}

\paragraph{Pusher target placement.}
Figure~\ref{fig:pusher-views} shows how target placement changes the required cylinder motion. The arm must first approach the cylinder from a suitable side and then maintain contact while pushing toward the target. The near-target image is drawn from PusherHard's target distribution; PusherHard Far restricts the target distribution as specified in Appendix~\ref{app:tasks}. These views help interpret why reaching the cylinder alone is insufficient for success.

\begin{figure}[ht]
\centering
\begin{minipage}{0.32\linewidth}
\centering\includegraphics[width=\linewidth]{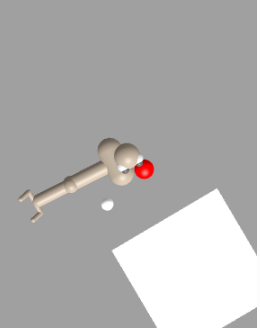}\\
PusherHard: near target
\end{minipage}\hspace{0.06\linewidth}
\begin{minipage}{0.32\linewidth}
\centering\includegraphics[width=\linewidth]{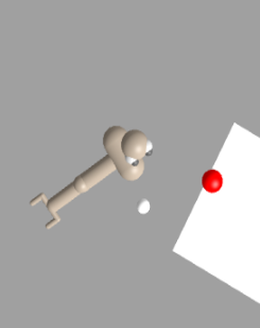}\\
PusherHard Far
\end{minipage}
\caption{Additional Pusher views. The white cylinder must be moved to the red target. Target placement changes the required approach and pushing direction. The near-target view is an example from PusherHard.}
\label{fig:pusher-views}
\end{figure}

\paragraph{Comparing waypoint scoring rules.}
Figure~\ref{fig:ranking} holds the encoder and sampled candidate sets fixed while changing the selection rule. An exact forward on-path waypoint is available in 77.7\% of draws. The gap between availability and selection therefore measures missed opportunities within the sampled set. Direction-agreement and latent-slack scoring select on-path candidates more frequently than the default inner product in this test. Evaluating these rules during online training would additionally measure how their choices change the subsequently collected experience.

\begin{figure}[ht]
\centering
\includegraphics[width=0.9\linewidth]{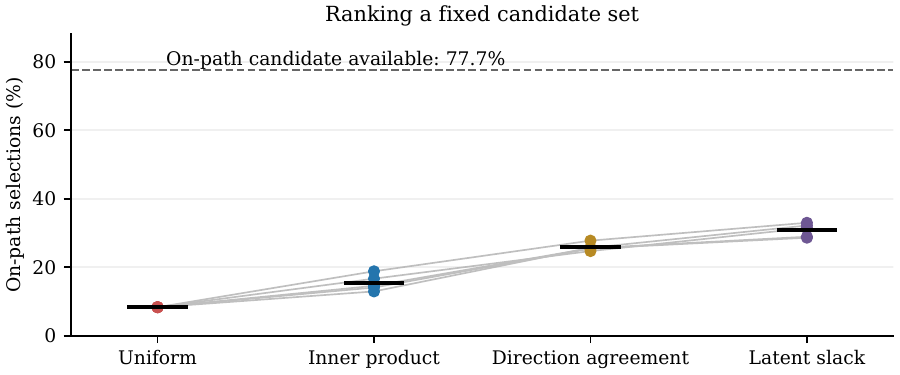}
\caption{Waypoint selection with five fixed U-maze encoders. Connected points use the same encoder and candidate draws; black marks show means. A forward on-path candidate is available in 77.7\% of draws, but the default inner-product rule selects one much less often. Ranking by direction agreement or latent slack improves this fixed-encoder test.}
\label{fig:ranking}
\end{figure}

\paragraph{Generating candidates beyond the visited-state pool.}
Figure~\ref{fig:candidate-sources} expands Table~\ref{tab:generator} into paired seed results. Each point averages success across the 20 evaluation checkpoints of a complete training run. The generator improves over the pool in four of five seeds in the two-route maze and all five in the rare-gateway maze. Projecting generated candidates onto visited states improves the mean in the two-route maze but reduces it in the rare-gateway maze. This difference is consistent with the layouts: projection can preserve useful route choices when visited states cover both routes, but can remove proposals beyond a poorly explored gateway.

\begin{figure}[ht]
\centering
\includegraphics[width=\linewidth]{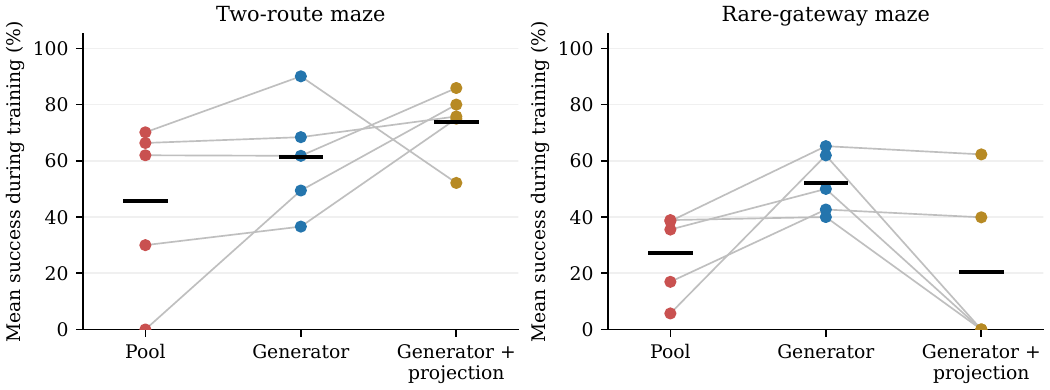}
\caption{Candidate-source comparison with all five seeds shown. Lines connect runs with the same seed; black marks show means. The generator improves mean success during training in both layouts. Nearest-state projection has different effects in the two mazes and removes the mean benefit in the rare-gateway maze. Each value averages the 20 evaluation checkpoints.}
\label{fig:candidate-sources}
\end{figure}
\clearpage

\section{Complete Benchmark Checkpoints}
\label{app:curves}
Figure~\ref{fig:curves} gives the complete nine-task learning curves. The following tables give the five-seed checkpoint evaluation. Nominal checkpoints are matched to the nearest available logged evaluation. Dashes indicate that a checkpoint lies beyond the task's training budget. Student-$t$ intervals are reported without clipping to the metric's bounds.

Table~\ref{tab:checkpoints-success} reports whether the goal region is reached at least once per episode. Reading across a task shows when performance improves and whether it persists. On PusherHard, DCP reaches 33.7\% at 25M steps, compared with 0.9\% for CRL; CRL later improves to 28.8\% at 100M, while DCP reaches 47.7\%. AntPush shows a different pattern: DCP's mean falls from 44.8\% at 75M to 42.7\% at 100M. CRL retains the higher mean at every reported AntSoccer checkpoint. Bold entries identify the largest mean, without implying statistical significance.

Table~\ref{tab:checkpoints-time} reports how long the agent stays near the goal. Comparing it with Table~\ref{tab:checkpoints-success} separates occasional arrival from sustained proximity. At the final checkpoint, DCP's PusherEasy success of 55.4\% accompanies 297.8 steps near the goal, compared with CRL's 4.4\% and 26.2 steps. On Humanoid U-Maze, DCP reaches the goal in 8.4\% of episodes but averages only 3.6 steps near it, indicating brief arrivals. The intervals describe variation across the five training seeds; the full curves in Figure~\ref{fig:curves} provide the evaluations between these selected checkpoints.
\begin{table}[!ht]
\centering\small
\caption{Success at least once (\%), mean $\pm$ reported 95\% confidence-interval half-width over five seeds. Bold marks the largest mean at each task/checkpoint.}
\label{tab:checkpoints-success}
\setlength{\tabcolsep}{3pt}
\input{tables/checkpoint_5seed_success_any}
\end{table}
\begin{table}[!ht]
\centering\small
\caption{Time near goal (steps per 1,000-step episode), mean $\pm$ reported 95\% confidence-interval half-width over five seeds.}
\label{tab:checkpoints-time}
\setlength{\tabcolsep}{3pt}
\input{tables/checkpoint_5seed_success}
\end{table}

\section{\rev{Direct-DCP and Parameter-Matched Comparisons}}
\label{app:direct-dcp}
\leavevmode\rev{Direct-DCP uses the DCP actor input $[s,d_\theta(s,g),r_\theta(s,g)]$ with the final goal $g$ as the collection target. Actor updates use achieved future states, as in DCP. All other settings, budgets, and seeds match the DCP runs. Figure~\ref{fig:curves} shows their learning curves alongside CRL, SSGC, and DCP.}

\paragraph{\rev{When is the final goal first reached?}}
\rev{Table~\ref{tab:first-reach} gives the first evaluation checkpoint at which each seed reaches 1\% success. On PusherEasy and Ant U-Maze, every seed of every method reaches it within 3M steps. On AntPush, no seed of any method reaches it before 18M steps, and three Direct-DCP seeds never do. On Humanoid U-Maze, every DCP seed reaches it within 13M steps; one Direct-DCP seed reaches it at 47M steps and the other four never do.}

\begin{table}[ht]
\centering\small
\caption{\rev{First evaluation checkpoint (millions of steps) at which each seed reaches 1\% success. A dash marks a seed that never reaches it. Seeds are ordered 1, 42, 123, 256, 1234.}}
\label{tab:first-reach}
\begin{tabular}{>{\revcol}l>{\revcol}l>{\revcol}l>{\revcol}l}
\toprule
Task & CRL & DCP & Direct-DCP \\
\midrule
AntPush & 29, 35, 41, 40, 40 & 45, 34, 21, 18, 48 & 27, 38, --, --, -- \\
Humanoid U-Maze & 19, 33, 6, 10, 7 & 10, 13, 3, 4, 12 & --, --, --, 47, -- \\
PusherEasy & 3, 2, 3, 2, 2 & 2, 2, 3, 2, 2 & 3, 2, 3, 2, 2 \\
Ant U-Maze & 2, 3, 2, 2, 2 & 2, 2, 2, 3, 2 & 3, 3, 2, 2, 3 \\
\bottomrule
\end{tabular}
\end{table}

\paragraph{\rev{Staying at the goal after reaching it.}}
\rev{For each evaluation episode that reaches the goal, we compute the fraction of the remaining steps spent in the goal region, and average over these episodes at the final checkpoint. Table~\ref{tab:hold} reports the mean over seeds with nonzero final success, together with the seed range. On the three Pusher tasks, DCP stays longer than Direct-DCP; on PusherHard and PusherHard Far, the seed ranges do not overlap. Humanoid U-Maze is omitted because few of its episodes reach the goal. Where CRL rarely succeeds, its fraction is computed from few episodes, which explains its wide ranges on the Pusher tasks.}

\begin{table}[ht]
\centering\small
\caption{\rev{Percentage of post-reach steps spent in the goal region at the final checkpoint. Means over seeds with nonzero final success; brackets give the range across those seeds, and $n$ the number of seeds when fewer than five.}}
\label{tab:hold}
\begin{tabular}{>{\revcol}l>{\revcol}l>{\revcol}l>{\revcol}l}
\toprule
Task & CRL & DCP & Direct-DCP \\
\midrule
AntMaze Big & 80 [78, 84] & 57 [51, 67] & 66 [56, 82] \\
AntMaze Hardest & 70 [62, 75] & 57 [48, 66] & 56 [39, 73] \\
Ant U-Maze & 87 [77, 93] & 66 [60, 77] & 86 [82, 89] \\
AntPush & 77 [66, 91] & 65 [51, 80] & 81 [81, 81], $n=2$ \\
AntSoccer & 49 [45, 52] & 43 [30, 54] & 47 [41, 50] \\
PusherEasy & 79 [32, 100] & 63 [48, 77] & 44 [29, 55] \\
PusherHard & 69 [46, 100] & 66 [52, 73] & 39 [31, 45] \\
PusherHard Far & 60 [38, 100], $n=4$ & 61 [52, 72] & 36 [28, 41] \\
\bottomrule
\end{tabular}
\end{table}

\paragraph{\rev{Statistical comparison.}}
\rev{Welch's $t$-tests over five seeds per method compare the final values of DCP and Direct-DCP. Direct-DCP's success is higher on PusherEasy ($p=0.024$), PusherHard ($p<0.001$), and AntSoccer ($p=0.006$). DCP's time near goal is higher on PusherHard Far ($p=0.033$), and Direct-DCP's is higher on AntSoccer ($p=0.037$). The remaining differences in either metric have $p>0.05$, including AntPush ($p=0.13$ for success, $0.061$ for time near goal) and Humanoid U-Maze ($p=0.074$ and $0.093$); on these tasks, Direct-DCP's failures are confined to some seeds.}

\paragraph{\rev{Parameter-matched raw-goal actor.}}
\rev{On PusherEasy, the CRL actor has 75,534 parameters and the DCP actor 91,406. Widening only the CRL actor's first hidden layer from 256 to 313 units gives 91,494 parameters; all other settings match the CRL runs. Across seeds 1, 42, 123, 256, and 1234, it reaches final success of 6.2\%, 2.7\%, 9.8\%, 2.3\%, and 2.7\% (mean 4.7\%) and time near goal of 24, 27, 57, 20, and 27 steps (mean 31). Default CRL with the same seeds reaches 2.3\%, 5.1\%, 3.1\%, 9.8\%, and 1.6\% (mean 4.4\%) and 23, 31, 31, 30, and 16 steps (mean 26).}

\paragraph{\rev{CRL on the Pusher tasks.}}
\rev{CRL's final success varies widely across seeds on the harder Pusher tasks: 0.4, 48.8, 1.2, 53.9, and 39.5\% on PusherHard and 7.0, 1.2, 7.8, 46.5, and 0.0\% on PusherHard Far. On PusherEasy, every seed stays below 10\%. DCP and Direct-DCP succeed in every seed on all three tasks, with the lowest seeds at 40.2\% and 41.8\%.}

%% file: tables/checkpoint_5seed_success_any.tex
\begin{tabular}{llccccc}
\toprule
Environment & Method & 10M & 25M & 50M & 75M & 100M \\
\midrule
AntMaze Big & DCP & \textbf{18.4 $\pm$ 5.4} & \textbf{23.8 $\pm$ 3.3} & \textbf{29.9 $\pm$ 2.9} & -- & -- \\
 & CRL & 15.1 $\pm$ 3.7 & 17.7 $\pm$ 5.3 & 18.3 $\pm$ 6.5 & -- & -- \\
 & SSGC & 18.3 $\pm$ 1.2 & 22.0 $\pm$ 2.9 & 23.3 $\pm$ 3.0 & -- & -- \\
\midrule
AntMaze Hardest & DCP & \textbf{9.5 $\pm$ 2.5} & \textbf{14.5 $\pm$ 2.1} & \textbf{16.1 $\pm$ 2.3} & -- & -- \\
 & CRL & 5.8 $\pm$ 2.5 & 10.5 $\pm$ 3.8 & 12.2 $\pm$ 5.6 & -- & -- \\
 & SSGC & 7.7 $\pm$ 1.6 & 13.2 $\pm$ 4.8 & 14.2 $\pm$ 8.0 & -- & -- \\
\midrule
Ant U-Maze & DCP & \textbf{40.4 $\pm$ 9.5} & \textbf{52.9 $\pm$ 6.8} & \textbf{55.0 $\pm$ 13.3} & -- & -- \\
 & CRL & 26.8 $\pm$ 4.5 & 40.7 $\pm$ 8.5 & 40.8 $\pm$ 10.2 & -- & -- \\
 & SSGC & 28.4 $\pm$ 9.6 & 43.1 $\pm$ 8.1 & 49.8 $\pm$ 4.2 & -- & -- \\
\midrule
Humanoid U-Maze & DCP & 4.6 $\pm$ 8.4 & 6.7 $\pm$ 11.4 & \textbf{13.4 $\pm$ 9.5} & \textbf{7.7 $\pm$ 9.2} & \textbf{8.4 $\pm$ 9.6} \\
 & CRL & \textbf{9.8 $\pm$ 11.5} & \textbf{7.0 $\pm$ 12.0} & 3.0 $\pm$ 8.5 & 0.0 $\pm$ 0.0 & 0.0 $\pm$ 0.0 \\
 & SSGC & 0.0 $\pm$ 0.0 & 0.0 $\pm$ 0.0 & 0.0 $\pm$ 0.0 & 2.5 $\pm$ 6.9 & 0.0 $\pm$ 0.0 \\
\midrule
AntPush & DCP & 0.0 $\pm$ 0.0 & \textbf{1.1 $\pm$ 2.8} & \textbf{29.1 $\pm$ 19.3} & \textbf{44.8 $\pm$ 8.0} & 42.7 $\pm$ 12.6 \\
 & CRL & 0.0 $\pm$ 0.0 & 0.0 $\pm$ 0.0 & 28.9 $\pm$ 6.3 & 38.2 $\pm$ 8.8 & 45.6 $\pm$ 24.1 \\
 & SSGC & 0.0 $\pm$ 0.0 & 0.0 $\pm$ 0.0 & 28.0 $\pm$ 32.1 & 40.8 $\pm$ 21.1 & \textbf{51.4 $\pm$ 15.5} \\
\midrule
PusherEasy & DCP & \textbf{2.8 $\pm$ 2.2} & \textbf{30.9 $\pm$ 35.0} & \textbf{55.4 $\pm$ 9.6} & -- & -- \\
 & CRL & 1.9 $\pm$ 0.9 & 2.2 $\pm$ 1.8 & 4.4 $\pm$ 4.1 & -- & -- \\
 & SSGC & 1.9 $\pm$ 0.9 & 7.2 $\pm$ 10.8 & 34.1 $\pm$ 38.5 & -- & -- \\
\midrule
PusherHard & DCP & \textbf{2.0 $\pm$ 1.9} & \textbf{33.7 $\pm$ 14.2} & \textbf{37.2 $\pm$ 11.7} & \textbf{45.6 $\pm$ 10.8} & \textbf{47.7 $\pm$ 5.8} \\
 & CRL & 0.3 $\pm$ 0.2 & 0.9 $\pm$ 0.6 & 2.0 $\pm$ 2.5 & 7.3 $\pm$ 12.8 & 28.8 $\pm$ 32.4 \\
 & SSGC & 0.2 $\pm$ 0.3 & 1.2 $\pm$ 1.4 & 8.4 $\pm$ 14.8 & 21.6 $\pm$ 26.4 & 27.8 $\pm$ 30.7 \\
\midrule
PusherHard Far & DCP & \textbf{1.1 $\pm$ 2.8} & \textbf{25.8 $\pm$ 17.3} & \textbf{47.4 $\pm$ 5.7} & -- & -- \\
 & CRL & 0.0 $\pm$ 0.0 & 1.6 $\pm$ 1.9 & 12.5 $\pm$ 24.0 & -- & -- \\
 & SSGC & 0.0 $\pm$ 0.0 & 0.2 $\pm$ 0.3 & 9.0 $\pm$ 15.1 & -- & -- \\
\midrule
AntSoccer & DCP & 2.6 $\pm$ 0.9 & 8.8 $\pm$ 2.9 & 12.9 $\pm$ 4.3 & -- & -- \\
 & CRL & \textbf{5.6 $\pm$ 3.6} & \textbf{14.9 $\pm$ 4.8} & \textbf{22.6 $\pm$ 3.4} & -- & -- \\
 & SSGC & 2.5 $\pm$ 1.3 & 7.0 $\pm$ 2.5 & 14.1 $\pm$ 3.6 & -- & -- \\
\bottomrule
\end{tabular}

%% file: tables/checkpoint_5seed_success.tex
\begin{tabular}{llccccc}
\toprule
Environment & Method & 10M & 25M & 50M & 75M & 100M \\
\midrule
AntMaze Big & DCP & 78.7 $\pm$ 37.1 & \textbf{94.6 $\pm$ 26.2} & \textbf{115.3 $\pm$ 14.6} & -- & -- \\
 & CRL & 75.3 $\pm$ 12.7 & 83.2 $\pm$ 29.4 & 89.2 $\pm$ 28.0 & -- & -- \\
 & SSGC & \textbf{91.9 $\pm$ 19.1} & 94.6 $\pm$ 20.5 & 99.8 $\pm$ 9.0 & -- & -- \\
\midrule
AntMaze Hardest & DCP & 30.3 $\pm$ 12.6 & \textbf{65.1 $\pm$ 12.9} & 63.2 $\pm$ 13.3 & -- & -- \\
 & CRL & 28.4 $\pm$ 15.9 & 39.9 $\pm$ 24.0 & 48.1 $\pm$ 23.7 & -- & -- \\
 & SSGC & \textbf{32.6 $\pm$ 8.1} & 59.0 $\pm$ 21.0 & \textbf{66.0 $\pm$ 41.4} & -- & -- \\
\midrule
Ant U-Maze & DCP & \textbf{195.3 $\pm$ 58.0} & 223.5 $\pm$ 27.9 & 238.1 $\pm$ 76.2 & -- & -- \\
 & CRL & 161.9 $\pm$ 33.1 & \textbf{233.2 $\pm$ 62.8} & \textbf{261.3 $\pm$ 77.2} & -- & -- \\
 & SSGC & 157.7 $\pm$ 65.3 & 230.4 $\pm$ 38.6 & 252.7 $\pm$ 52.1 & -- & -- \\
\midrule
Humanoid U-Maze & DCP & \textbf{1.3 $\pm$ 2.5} & \textbf{4.9 $\pm$ 9.5} & \textbf{3.4 $\pm$ 2.5} & \textbf{1.4 $\pm$ 1.9} & \textbf{3.6 $\pm$ 4.5} \\
 & CRL & 1.0 $\pm$ 1.7 & 1.4 $\pm$ 3.0 & 0.9 $\pm$ 2.5 & 0.0 $\pm$ 0.0 & 0.0 $\pm$ 0.0 \\
 & SSGC & 0.0 $\pm$ 0.0 & 0.0 $\pm$ 0.0 & 0.0 $\pm$ 0.0 & 0.4 $\pm$ 1.1 & 0.0 $\pm$ 0.0 \\
\midrule
AntPush & DCP & 0.0 $\pm$ 0.0 & \textbf{3.8 $\pm$ 10.4} & \textbf{131.8 $\pm$ 113.5} & \textbf{190.9 $\pm$ 71.6} & 181.7 $\pm$ 95.6 \\
 & CRL & 0.0 $\pm$ 0.0 & 0.0 $\pm$ 0.0 & 95.9 $\pm$ 42.6 & 122.7 $\pm$ 30.1 & 143.4 $\pm$ 105.6 \\
 & SSGC & 0.0 $\pm$ 0.0 & 0.0 $\pm$ 0.0 & 103.2 $\pm$ 114.9 & 182.9 $\pm$ 83.9 & \textbf{199.1 $\pm$ 93.3} \\
\midrule
PusherEasy & DCP & 18.2 $\pm$ 9.0 & \textbf{148.2 $\pm$ 179.5} & \textbf{297.8 $\pm$ 90.9} & -- & -- \\
 & CRL & \textbf{18.8 $\pm$ 9.3} & 20.3 $\pm$ 15.1 & 26.2 $\pm$ 8.4 & -- & -- \\
 & SSGC & \textbf{18.8 $\pm$ 9.3} & 40.0 $\pm$ 49.3 & 180.7 $\pm$ 183.4 & -- & -- \\
\midrule
PusherHard & DCP & \textbf{9.3 $\pm$ 4.2} & \textbf{161.5 $\pm$ 84.4} & \textbf{185.7 $\pm$ 45.4} & \textbf{285.3 $\pm$ 103.5} & \textbf{280.9 $\pm$ 28.5} \\
 & CRL & 3.1 $\pm$ 2.2 & 9.4 $\pm$ 5.5 & 13.1 $\pm$ 9.9 & 36.2 $\pm$ 66.8 & 125.7 $\pm$ 136.7 \\
 & SSGC & 2.3 $\pm$ 2.7 & 9.4 $\pm$ 8.1 & 54.0 $\pm$ 99.4 & 133.1 $\pm$ 161.3 & 170.5 $\pm$ 186.0 \\
\midrule
PusherHard Far & DCP & \textbf{2.5 $\pm$ 4.8} & \textbf{124.1 $\pm$ 88.4} & \textbf{260.7 $\pm$ 66.8} & -- & -- \\
 & CRL & 0.0 $\pm$ 0.0 & 10.7 $\pm$ 15.9 & 48.4 $\pm$ 80.4 & -- & -- \\
 & SSGC & 0.0 $\pm$ 0.0 & 1.6 $\pm$ 2.7 & 49.0 $\pm$ 82.3 & -- & -- \\
\midrule
AntSoccer & DCP & 11.2 $\pm$ 8.8 & 33.6 $\pm$ 18.7 & 42.4 $\pm$ 22.1 & -- & -- \\
 & CRL & \textbf{20.8 $\pm$ 10.5} & \textbf{46.3 $\pm$ 13.8} & \textbf{76.4 $\pm$ 7.7} & -- & -- \\
 & SSGC & 9.4 $\pm$ 6.9 & 29.8 $\pm$ 12.6 & 48.2 $\pm$ 31.9 & -- & -- \\
\bottomrule
\end{tabular}

%% file: refs.bib
@inproceedings{eysenbach2022crl,
  author = {Eysenbach, Benjamin and Zhang, Tianjun and Levine, Sergey and Salakhutdinov, Russ R},
  title = {Contrastive Learning as Goal-Conditioned Reinforcement Learning},
  booktitle = {Advances in Neural Information Processing Systems},
  volume = {35}, pages = {35603--35620}, year = {2022},
  url = {https://proceedings.neurips.cc/paper_files/paper/2022/file/e7663e974c4ee7a2b475a4775201ce1f-Paper-Conference.pdf}
}

@inproceedings{bortkiewicz2025jaxgcrl,
  author = {Bortkiewicz, Micha{\l} and Pa{\l}ucki, W{\l}adys{\l}aw and Myers, Vivek and Dziarmaga, Tadeusz and Arczewski, Tomasz and Kuci{\'n}ski, {\L}ukasz and Eysenbach, Benjamin},
  title = {Accelerating Goal-Conditioned Reinforcement Learning Algorithms and Research},
  booktitle = {International Conference on Learning Representations}, year = {2025},
  url = {https://arxiv.org/abs/2408.11052}
}

@inproceedings{wang2025depth,
  author = {Wang, Kevin and Javali, Ishaan and Bortkiewicz, Micha{\l} and Trzcinski, Tomasz and Eysenbach, Benjamin},
  title = {1000 Layer Networks for Self-Supervised {RL}: Scaling Depth Can Enable New Goal-Reaching Capabilities},
  booktitle = {Advances in Neural Information Processing Systems},
  volume = {38}, pages = {157643--157670}, year = {2025},
  url = {https://proceedings.neurips.cc/paper_files/paper/2025/file/e74ee34cc0f2d0780f34ee77d8fba25b-Paper-Conference.pdf}
}

@article{korniak2026chunks,
  author = {Korniak, Michal and Dybek, Kamil and Eysenbach, Benjamin and Bagatella, Marco and Bortkiewicz, Micha{\l}},
  title = {Three Steps at a Time: Learning Representations from Action Sequences in Contrastive {RL}},
  journal = {arXiv preprint arXiv:2608.30640}, year = {2026},
  url = {https://arxiv.org/abs/2608.30640}
}

@inproceedings{park2024hilp,
  author = {Park, Seohong and Kreiman, Tobias and Levine, Sergey},
  title = {Foundation Policies with {Hilbert} Representations},
  booktitle = {Proceedings of the 41st International Conference on Machine Learning},
  series = {Proceedings of Machine Learning Research}, volume = {235},
  pages = {39737--39761}, year = {2024},
  url = {https://proceedings.mlr.press/v235/park24g.html}
}

@inproceedings{park2024metra,
  author = {Park, Seohong and Rybkin, Oleh and Levine, Sergey},
  title = {{METRA}: Scalable Unsupervised {RL} with Metric-Aware Abstraction},
  booktitle = {International Conference on Learning Representations}, year = {2024},
  url = {https://arxiv.org/abs/2310.08887}
}

@inproceedings{devin2019plan,
  author = {Devin, Coline and Geng, Daniel and Abbeel, Pieter and Darrell, Trevor and Levine, Sergey},
  title = {Plan Arithmetic: Compositional Plan Vectors for Multi-Task Control},
  booktitle = {Advances in Neural Information Processing Systems}, year = {2019},
  url = {https://arxiv.org/abs/1910.14033}
}

@inproceedings{haarnoja2018sac,
  author = {Haarnoja, Tuomas and Zhou, Aurick and Abbeel, Pieter and Levine, Sergey},
  title = {Soft Actor-Critic: Off-Policy Maximum Entropy Deep Reinforcement Learning with a Stochastic Actor},
  booktitle = {Proceedings of the 35th International Conference on Machine Learning},
  volume = {80}, pages = {1861--1870}, year = {2018},
  url = {https://proceedings.mlr.press/v80/haarnoja18b.html}
}

@inproceedings{freeman2021brax,
  author = {Freeman, Daniel and Frey, Erik and Raichuk, Anton and Girgin, Sertan and Mordatch, Igor and Bachem, Olivier},
  title = {Brax---A Differentiable Physics Engine for Large Scale Rigid Body Simulation},
  booktitle = {Proceedings of the Neural Information Processing Systems Track on Datasets and Benchmarks},
  volume = {1}, year = {2021},
  url = {https://datasets-benchmarks-proceedings.neurips.cc/paper_files/paper/2021/file/d1f491a404d6854880943e5c3cd9ca25-Paper-round1.pdf}
}

@inproceedings{eysenbach2024interpolation,
  author = {Eysenbach, Benjamin and Myers, Vivek and Salakhutdinov, Ruslan and Levine, Sergey},
  title = {Inference via Interpolation: Contrastive Representations Provably Enable Planning and Inference},
  booktitle = {Advances in Neural Information Processing Systems},
  volume = {37}, pages = {58901--58928}, year = {2024},
  doi = {10.52202/079017-1878},
  url = {https://arxiv.org/abs/2403.04082}
}

@inproceedings{myers2025horizon,
  author = {Myers, Vivek and Ji, Catherine and Eysenbach, Benjamin},
  title = {Horizon Generalization in Reinforcement Learning},
  booktitle = {International Conference on Learning Representations}, year = {2025},
  url = {https://arxiv.org/abs/2501.02709}
}

@inproceedings{myers2024distances,
  author = {Myers, Vivek and Zheng, Chongyi and Dragan, Anca and Levine, Sergey and Eysenbach, Benjamin},
  title = {Learning Temporal Distances: Contrastive Successor Features Can Provide a Metric Structure for Decision-Making},
  booktitle = {Proceedings of the 41st International Conference on Machine Learning},
  series = {Proceedings of Machine Learning Research}, volume = {235},
  pages = {37076--37096}, year = {2024},
  url = {https://proceedings.mlr.press/v235/myers24a.html}
}

@inproceedings{wang2023qrl,
  author = {Wang, Tongzhou and Torralba, Antonio and Isola, Phillip and Zhang, Amy},
  title = {Optimal Goal-Reaching Reinforcement Learning via Quasimetric Learning},
  booktitle = {Proceedings of the 40th International Conference on Machine Learning},
  series = {Proceedings of Machine Learning Research}, volume = {202},
  pages = {36411--36430}, year = {2023},
  url = {https://proceedings.mlr.press/v202/wang23al.html}
}

@inproceedings{liu2025single,
  title={A Single Goal Is All You Need: Skills and Exploration Emerge from Contrastive {RL} Without Rewards, Demonstrations, or Subgoals},
  author={Liu, Grace and Tang, Michael and Eysenbach, Benjamin},
  booktitle={International Conference on Learning Representations},
  year={2025},
  url={https://proceedings.iclr.cc/paper_files/paper/2025/file/c34ddd05eb089991f06f3c5dc36836e0-Paper-Conference.pdf}
}

@inproceedings{park2026dual,
  title={Dual Goal Representations},
  author={Park, Seohong and Mann, Deepinder and Levine, Sergey},
  booktitle={International Conference on Learning Representations},
  year={2026},
  url={https://openreview.net/forum?id=aMKFTidLSM}
}

@misc{vezhnevets2017feudal,
      title={FeUdal Networks for Hierarchical Reinforcement Learning}, 
      author={Alexander Sasha Vezhnevets and Simon Osindero and Tom Schaul and Nicolas Heess and Max Jaderberg and David Silver and Koray Kavukcuoglu},
      year={2017},
      eprint={1703.01161},
      archivePrefix={arXiv},
      primaryClass={cs.AI},
      url={https://arxiv.org/abs/1703.01161}, 
}

@inproceedings{nachum2018data,
author = {Nachum, Ofir and Gu, Shixiang and Lee, Honglak and Levine, Sergey},
title = {Data-efficient hierarchical reinforcement learning},
year = {2018},
publisher = {Curran Associates Inc.},
address = {Red Hook, NY, USA},
booktitle = {Proceedings of the 32nd International Conference on Neural Information Processing Systems},
pages = {3307–3317},
numpages = {11},
location = {Montr{\'e}al, Canada},
series = {NIPS'18}
}

@inproceedings{
park2022lipschitz,
title={Lipschitz-constrained Unsupervised Skill Discovery},
author={Seohong Park and Jongwook Choi and Jaekyeom Kim and Honglak Lee and Gunhee Kim},
booktitle={International Conference on Learning Representations},
year={2022},
url={https://openreview.net/forum?id=BGvt0ghNgA}
}

@inproceedings{park2023hiql,
 author = {Park, Seohong and Ghosh, Dibya and Eysenbach, Benjamin and Levine, Sergey},
 booktitle = {Advances in Neural Information Processing Systems},
 doi = {10.52202/075280-1515},
 editor = {A. Oh and T. Naumann and A. Globerson and K. Saenko and M. Hardt and S. Levine},
 pages = {34866--34891},
 publisher = {Curran Associates, Inc.},
 title = {HIQL: Offline Goal-Conditioned RL with Latent States as Actions},
 url = {https://proceedings.neurips.cc/paper_files/paper/2023/file/6d7c4a0727e089ed6cdd3151cbe8d8ba-Paper-Conference.pdf},
 volume = {36},
 year = {2023}
}

@inproceedings{
hartikainen2020dynamical,
title={Dynamical Distance Learning for Semi-Supervised and Unsupervised Skill Discovery},
author={Kristian Hartikainen and Xinyang Geng and Tuomas Haarnoja and Sergey Levine},
booktitle={International Conference on Learning Representations},
year={2020},
url={https://openreview.net/forum?id=H1lmhaVtvr}
}

@InProceedings{liu2023metric,
  author = {Bo Liu and Yihao Feng and Qiang Liu and Peter Stone},
  title = {Metric Residual Networks for Sample Efficient Goal-Conditioned Reinforcement Learning},
  booktitle = {Thirty-Seventh AAAI Conference on Artificial Intelligence (AAAI)},
  location = {Washington, DC, US},
  month = {Februray},
  year = {2023},
}

@inproceedings{myers2025offline,
title = "Offline Goal-conditioned Reinforcement Learning with Quasimetric Representations",
author = "Vivek Myers and Zheng, \{Bill Chunyuan\} and Benjamin Eysenbach and Sergey Levine",
note = "Publisher Copyright: {\textcopyright} 2025 Neural information processing systems foundation. All rights reserved.; 39th Conference on Neural Information Processing Systems, NeurIPS 2025 ; Conference date: 02-12-2025 Through 07-12-2025",
year = "2025",
language = "English (US)",
series = "Advances in Neural Information Processing Systems",
publisher = "Neural information processing systems foundation",
pages = "22292--22317",
booktitle = "39th Conference on Neural Information Processing Systems, NeurIPS 2025",
}

@inproceedings{
zheng2023stabilizing,
title={Stabilizing Contrastive {RL}: Techniques for Robotic Goal Reaching from Offline Data},
author={Chongyi Zheng and Benjamin Eysenbach and Homer Walke and Patrick Yin and Kuan Fang and Ruslan Salakhutdinov and Sergey Levine},
booktitle={Intrinsically-Motivated and Open-Ended Learning Workshop @NeurIPS2023},
year={2023},
url={https://openreview.net/forum?id=IAX12qUjYk}
}

@inproceedings{eysenbach2019search,
 author = {Eysenbach, Ben and Salakhutdinov, Russ R and Levine, Sergey},
 booktitle = {Advances in Neural Information Processing Systems},
 editor = {H. Wallach and H. Larochelle and A. Beygelzimer and F. d\textquotesingle Alch\'{e}-Buc and E. Fox and R. Garnett},
 pages = {},
 publisher = {Curran Associates, Inc.},
 title = {Search on the Replay Buffer: Bridging Planning and Reinforcement Learning},
 url = {https://proceedings.neurips.cc/paper_files/paper/2019/file/5c48ff18e0a47baaf81d8b8ea51eec92-Paper.pdf},
 volume = {32},
 year = {2019}
}

@misc{ecoffet2021first,
      title={Go-Explore: a New Approach for Hard-Exploration Problems}, 
      author={Adrien Ecoffet and Joost Huizinga and Joel Lehman and Kenneth O. Stanley and Jeff Clune},
      year={2021},
      eprint={1901.10995},
      archivePrefix={arXiv},
      primaryClass={cs.LG},
      url={https://arxiv.org/abs/1901.10995}, 
}

@inproceedings{
zhou2025flattening,
title={Flattening Hierarchies with Policy Bootstrapping},
author={John Luoyu Zhou and Jonathan Kao},
booktitle={The Thirty-ninth Annual Conference on Neural Information Processing Systems},
year={2025},
url={https://openreview.net/forum?id=KaD2Dw8Ahz}
}

@InProceedings{jurgenson2020sub,
  title = 	 {Sub-Goal Trees a Framework for Goal-Based Reinforcement Learning},
  author =       {Jurgenson, Tom and Avner, Or and Groshev, Edward and Tamar, Aviv},
  booktitle = 	 {Proceedings of the 37th International Conference on Machine Learning},
  pages = 	 {5020--5030},
  year = 	 {2020},
  editor = 	 {III, Hal Daumé and Singh, Aarti},
  volume = 	 {119},
  series = 	 {Proceedings of Machine Learning Research},
  month = 	 {13--18 Jul},
  publisher =    {PMLR},
  url = 	 {https://proceedings.mlr.press/v119/jurgenson20a.html}
}
